\documentclass[letterpaper]{article}

\usepackage{aaai2027}
\usepackage{natbib}
\usepackage{url}            
\usepackage{booktabs}       
\usepackage{amsfonts}       
\usepackage{nicefrac}       
\usepackage{microtype}      
\usepackage{subcaption}
\usepackage{algorithm}
\usepackage{algorithmicx}
\usepackage{tabularx}
\usepackage{amsmath}
\usepackage{amssymb}
\usepackage{mathtools}
\usepackage{amsthm}
\usepackage{algpseudocode}
\usepackage{graphicx}
\usepackage[table]{xcolor}
\usepackage{tcolorbox}
\usepackage{multirow}
\usepackage{longtable}
\usepackage{enumitem}
\usepackage{nameref}
\theoremstyle{plain}
\newtheorem{theorem}{Theorem}
\newtheorem{proposition}[theorem]{Proposition}

\theoremstyle{definition}

\theoremstyle{remark}
\newtheorem{remark}[theorem]{Remark}

\newcommand{\SE}{\mathtt{SE}}

\newcommand{\EE}{\mathtt{IGE}}

\newcommand{\SGE}{\mathtt{LGU}}
\newcommand{\LAG}{\mathrm{ImG}}

\newcommand{\CG}{\mathrm{InG}}
\newcommand{\CS}{\mathtt{InS}}
\newcommand{\eb}[2]{#1\,$\pm$\,#2}

\title{Beyond Semantic Equivalence: Logical Graphs for Large Language Model Uncertainty Quantification}
\author{
    Yanni Dong,\textsuperscript{\rm 1}\thanks{Supported by the postdoctoral program at Guangzhou Jiayi Software Technology Co., Ltd. and Nansha IT Park.}\equalcontrib
    Minghua Liu,\textsuperscript{\rm 2}\equalcontrib
    Meilin Zhu,\textsuperscript{\rm 2}
    Xiaowei Huang,\textsuperscript{\rm 3}
    Lijun Zhang\textsuperscript{\rm 2}\thanks{Corresponding author.}
}
\affiliations{
    \textsuperscript{\rm 1}Institute of Intelligent Software, Guangzhou, China\\
    \textsuperscript{\rm 2}Institute of Software, Chinese Academy of Sciences, Beijing, China\\
    \textsuperscript{\rm 3}University of Liverpool, Liverpool, UK\\
    yannidong@outlook.com, liuminghua24@mails.ucas.ac.cn, martial0296@gmail.com, xiaowei.huang@liverpool.ac.uk, zhanglj@ios.ac.cn
}

\begin{document}
\maketitle
\begin{abstract}
 Large Language Models often produce confidently stated yet unreliable outputs, posing critical challenges for deployment in safety-sensitive applications. Existing uncertainty metrics such as semantic entropy capture agreement at the level of semantic equivalence, but largely ignore the {logical relationships} between distinct answers. As a result, they tend to overestimate uncertainty and falsely flag hallucinations in settings where generated responses are diverse in form yet logically compatible (e.g., differing only in granularity or specificity). We propose {Logical Graph Uncertainty (LGU)}, a framework that explicitly models implication and incompatibility among answers. 
 LGU aggregates probability mass along entailment chains onto the most specific hypotheses the answers support, measures the entropy of the resulting distribution, and penalizes mutual incompatibility among those hypotheses. Across multiple question-answering benchmarks and model families, LGU ranks first on average among existing uncertainty measures, with its largest gains---up to {+7.1\% AUROC} and {+3.5\% AUARC} over semantic entropy---on questions whose sampled answers are logically structured.
\end{abstract}

\section{Introduction}

\subsection{Uncertainty Quantification for Large Language Models} 
\paragraph{Large Language Model Hallucinations.} The remarkable capabilities of Large Language Models (LLMs) have been demonstrated across diverse natural language processing tasks, from question answering and text generation to complex reasoning and code synthesis \cite{llm-code-synthesis}, enabling their widespread adoption across numerous domains including software development \cite{llm-software-1,llm-software-2}, autonomous driving \cite{llm-ad} and education \cite{llm-edu}. However, despite their impressive performance, LLMs exhibit a critical limitation: they produce responses with varying degrees of uncertainty while maintaining seemingly equal confidence in their outputs. This uncertainty manifests in multiple forms, including factual inaccuracies \cite{llm-fact}, inconsistent reasoning patterns \cite{llm-reasoning-inconsistency} and hallucinated content \cite{hallucination-survey,uq4hallucination}, where successive generated responses can be mutually contradictory or semantically and logically unrelated. 
Reliable uncertainty estimates support informed decision-making and model calibration, and serve as a foundational component for hallucination detection \cite{uncertainty-survey1,uncertainty-survey2}, flagging unreliable outputs before they propagate through application pipelines. In this work we focus on {epistemic} (model-side, semantic) uncertainty in LLM answers---arising from the model's own knowledge and reasoning---rather than {aleatoric} uncertainty inherent to ambiguous queries.

\paragraph{Semantic-Level Measures.} Current state-of-the-art approaches to Uncertainty Quantification (UQ) in LLMs can be broadly categorized into several paradigms. Token-level methods \cite{prompt1,prompt2} focus on the distributional confidence of individual token predictions, typically leveraging softmax probabilities or token-level entropy measures. Sequence-level approaches \cite{sequence1,shapley} aggregate token-wise uncertainties or employ ensemble methods to assess overall response reliability. Two key semantic-level measures are {Semantic Entropy} (SE) \cite{se} and {Kernel Language Entropy (KLE)} \cite{kle}. Given an input $x$, sampled answers are grouped into $k$ paraphrase-equivalent semantic clusters $\mathcal{C}=\{c_1,\dots,c_k\}$, and each cluster receives probability mass $p(c_j\mid x)$ by summing the probabilities of its members. Specifically, SE computes
$
\SE(x)= -\sum_{j=1}^{k} p(c_j\mid x)\log p(c_j\mid x).$ Then KLE replaces semantic grouping with a kernel $K\in\mathbb{R}^{k\times k}$ over clusters and measures uncertainty via the von Neumann entropy
$
\mathtt{KLE}(x)= -\mathrm{Tr}[K\log K]
= -\sum_{i=1}^{k}\lambda_i\log \lambda_i,
$
where $\{\lambda_i\}$ are eigenvalues of a unit-trace positive semidefinite $K$, typically introducing additional kernel hyperparameters.

\begin{figure}[htbp]
    \centering
\includegraphics[width=\linewidth]{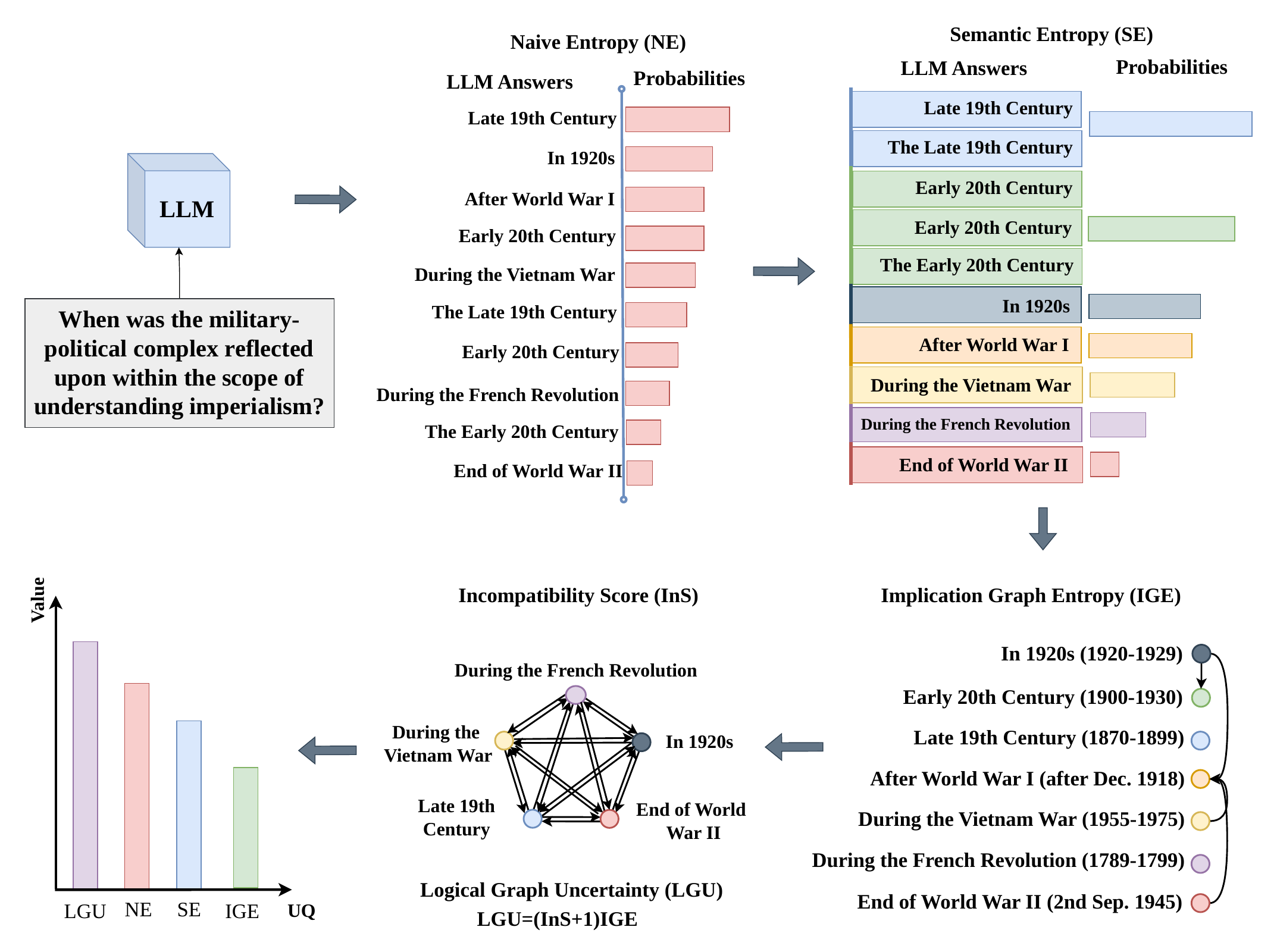}
    \caption{\textbf{Overview Comparison of Uncertainty Measures on a Question Answering Example.} From NE, SE to IGE, IGE aggregates probability mass along implication chains onto root hypotheses, and LGU scales it by InS over those roots, yielding $\mathtt{LGU}=(\mathtt{InS}+1)\mathtt{IGE}$. Here a larger LGU is desirable: it flags that the model spreads mass across mutually incompatible root hypotheses.}
    \label{fig:comparison}
\end{figure}
\paragraph{Limitation of Semantic Similarity.}
Despite their advances, semantic UQ methods judge only whether answers are {paraphrase-equivalent}, ignoring logical relationships among non-identical responses. For open-ended questions, multiple correct answers may exist that are not paraphrases yet are logically related; some entail others, some are compatible refinements, and some are genuinely mutually exclusive. SE assigns high entropy because the model generates diverse valid answers, conflating diversity with uncertainty.
Figure~\ref{fig:comparison} illustrates this limitation. Naive Entropy (NE)\footnote{Throughout this paper, we use the NE baseline as implemented in \cite{se_nature}.} treats raw strings as independent categories, inflating uncertainty under superficial variation. SE mitigates paraphrastic variation {but fails when answers are logically nested}: ``in the 1920s'' implies ``in the early 20th century'', yet treating them as separate clusters double-counts probability mass --- the failure mode our method addresses.



\subsection{Our Logical Graph-Based Uncertainty Framework}
We propose a {logical graph-based} uncertainty framework with two quantities: {Implication Graph Entropy (IGE)} and {Logical Graph Uncertainty (LGU)}. We construct an {Implication Graph} (ImG) whose directed edges encode logical implication relations among candidate answers under a shared truth-conditional standard (i.e., within the question-induced hypothesis space). Adopting graph-entropy principles as a measure of structural complexity~\cite{DEHMER201157}, we aggregate probability mass along implication chains, pushing mass from more general statements toward the more specific refinements that entail them, and compute entropy over the resulting distribution on the most specific (logically strongest), non-redundant hypotheses (roots). Because reachability sets may overlap, root scores are normalized to form a valid probability distribution. IGE thus measures how probability mass is distributed over distinct most-specific hypotheses, independent of their mutual consistency.

On top of ImG, we build an {Incompatibility Graph} (InG) over the most specific (root) hypotheses and define LGU by penalizing dense patterns of mutual inconsistency. Unlike implication-based aggregation, LGU asks: how much probability mass is spread over hypotheses that cannot simultaneously be true? This distinction is illustrated in Figure~\ref{fig:comparison}. When a model's fine-grained refinements are compatible (e.g., specific years that all imply ``in the 1920's''), aggregating mass along implication chains collapses them onto a single root and IGE stays low---correctly signalling low uncertainty. When instead the model distributes mass over mutually exclusive roots (e.g., {1920} vs.\ {1926} vs.\ {1929}), IGE alone can understate the conflict, whereas LGU becomes high because the roots form a densely incompatible set, exposing epistemic conflict that NE and SE may miss.

To quantify incompatibility, we model InG as a directed graph and define the Incompatibility Score (InS) using graph-density indices---edge density (and a weighted variant), average in-/out-degree, and the Estrada index---which we compare experimentally. Edge density yields a size-normalized score in $[0,1]$, delivers strong {Area Under the Receiver Operating Characteristic curve} (AUROC; $\uparrow$) overall, and achieves the lowest {Expected Calibration Error} (ECE; $\downarrow$) in our experiments; we adopt it as the default.

\begin{table}[htbp]
    \centering
    \vspace{-3mm}
    \caption{\textbf{Computational Complexity and Structural Exploitation of Four Semantic Uncertainty Measures.} Here $n$ is the number of answers.}
    \label{tab:uq-comparison}
    \resizebox{\columnwidth}{!}{
    \begin{tabular}{lcccccc}
        \toprule
        Method & Algorithm Complexity & Semantic Similarity  & Logical Relation & Extra Parameter& Domain \\
        \midrule
        NE        & $\boldsymbol{\mathcal{O}(n)}$
                                 & No & No    & \textbf{No}& [0,$\log n$]  \\
        SE        & $\mathcal{O}(n^2)$
                                 & \textbf{Yes} & No    & \textbf{No}& [0,$\log n$]   \\
        KLE  & $\mathcal{O}(n^3)$
                                 & \textbf{Yes}  & No  &Yes& [0,$\log n$]  \\
        \textbf{LGU}& $\mathcal{O}(n^2)$
                                 & \textbf{Yes}  & Yes & \textbf{No} & [0,$(1+\CS)\log n$]  \\
        \bottomrule
    \end{tabular}%
    }
    \vspace{-2mm}
\end{table}

Table~\ref{tab:uq-comparison} compares LGU with NE, SE, and KLE. LGU alone incorporates explicit logical relations without extra hyperparameters, stays $\mathcal{O}(n^2)$ (vs.\ KLE's $\mathcal{O}(n^3)$), and its wider range $[0,(1+\CS)\log n]$ gives finer granularity to separate benign diversity from genuine conflict.

\paragraph{Evaluation.} We validate LGU on four Question Answering (QA) benchmarks---TriviaQA, SQuAD, BioASQ, and Natural Questions (NQ)---for hallucination discrimination and selective prediction. 
Across datasets, LGU attains the best average rank against semantic- and token-level baselines in both ranking and calibration, with the largest gains on questions that elicit logically structured answer sets.
Full quantitative results are reported in Section~\nameref{sec:exp}.

To summarize, our contributions towards more precise UQ in LLMs are as follows:
\begin{tcolorbox}[
    colback=gray!2,
    colframe=black!75,
    fonttitle=\bfseries\large,
    title=Contributions,
    arc=0.5mm,
    boxrule=0.6pt,
    left=8pt, right=8pt, top=8pt, bottom=8pt
]
\begin{itemize}
  [leftmargin=*, itemsep=4pt, parsep=2pt]
    \item \textbf{Logical Graph Reformulation.} We recast LLM uncertainty as
    structured reasoning over the answer space---asking not ``do these answers
    agree?'' but ``what  do they imply, and do they contradict?''---via an ImG that propagates mass toward   the most specific
\end{itemize}
\end{tcolorbox}

\begin{tcolorbox}[
    colback=gray!2,
    colframe=black!75,
    arc=0.5mm,
    boxrule=0.6pt,
    left=8pt, right=8pt, top=8pt, bottom=8pt
]
\quad hypotheses and an InG that penalizes mutual conflict.
\begin{itemize}[leftmargin=*, itemsep=4pt, parsep=2pt]
   
    \item \textbf{Parameter-Free Graph Uncertainty.} We instantiate this as
    IGE and LGU: closed-form, hyperparameter-free scores computable in
    $\mathcal{O}(n^2)$ time, with no tuning, supervision, or oracle labels.
    \item \textbf{Empirical Study.} Across four QA datasets and eight backbones
    (7B--32B), in white- and black-box settings, 
    LGU attains the best average rank on AUROC, Area Under the Accuracy-Rejection Curve (AUARC), and ECE, with its largest gains over SE---up to $+7.1\%$ AUROC and $+3.5\%$ AUARC---on questions whose answers are logically structured.
\end{itemize}
\end{tcolorbox}


\section{Logical Graph Measures}
\label{sec:lge}

We model the set of LLM answers as {logical graphs}: answers are grouped into equivalence classes (clusters), and edges encode implication and incompatibility among these clusters. Based on this structure, we define two uncertainty measures. IGE aggregates probability mass along implication chains and computes entropy over the most specific clusters (roots). A graph-density-style InS among these roots is employed by LGU to further scale IGE, penalizing mutually exclusive most-specific hypotheses more strongly than compatible semantic diversity.


\subsection{Logical Graphs}
Before introducing the two graphs, ImG and InG, we establish key terminology and notation.

\paragraph{Preliminaries.} A \textbf{directed graph} is an ordered pair $G=(\mathcal{V},\mathcal{E})$,
where $\mathcal{V}$ is a set of vertices and $\mathcal{E}\subseteq \mathcal{V}\times \mathcal{V}$
is a set of ordered pairs called \textbf{directed edges}. 
An \textbf{empty graph} is a graph with $\mathcal{E}=\emptyset$.
A \textbf{self-loop} is an edge of the form $(v,v)$.
Throughout this paper, all graphs are assumed to be \textbf{directed graphs without self-loops}. For a vertex $v\in\mathcal{V}$, its \textbf{in-neighborhood} and \textbf{out-neighborhood} are defined as
$N^{-}(v)=\{u\in\mathcal{V}:(u,v)\in\mathcal{E}\}$ and
$N^{+}(v)=\{u\in\mathcal{V}:(v,u)\in\mathcal{E}\}$, respectively.
The \textbf{in-degree} and \textbf{out-degree} of $v$ are
$d^{-}(v)= |N^{-}(v)|$ and $d^{+}(v)= |N^{+}(v)|$. A \textbf{directed cycle} is a sequence of {distinct} vertices $(v_0,\dots,v_{k-1})$
such that $(v_i,v_{i+1})\in \mathcal{E}$ for $i=0,\dots,k-2$ and $(v_{k-1},v_0)\in \mathcal{E}$.
A directed graph is a \textbf{directed acyclic graph (DAG)} if it contains no directed cycles.
A vertex $r\in\mathcal{V}$ is a \textbf{root} if no edge terminates at $r$, i.e., $d^{-}(r)=0$.

\begin{remark}\label{rem:root}
A finite DAG must have at least one root. Indeed, if an arbitrary vertex $v_0$ is not a root, one can construct a sequence of vertices by picking $v_{i+1}$ such that $(v_{i+1},v_i)\in\mathcal{E}$ whenever $v_i$ is not a root. Since $\mathcal{V}$ is finite, this sequence must eventually repeat a vertex, which creates a directed cycle and contradicts acyclicity. Hence, the process must terminate at a root.
\end{remark}

\paragraph{Implication Graph.}
Given a query $x$ and an LLM $L$, we sample multiple candidate answers. Two answers $a,b$ are {(logically) equivalent} w.r.t.\ $x$ if, in every possible situation compatible with $x$, (i) $a$ true implies $b$ true and (ii) $b$ true implies $a$ true. We write $a\equiv b$. Each equivalence class forms an {equivalent cluster} $c$.
The ImG is defined as
\(
\LAG(x\mid L)=(\mathcal{V},\mathcal{E}),
\)
where
$
\mathcal{V}=\{\, c_i \mid c_i \text{ is an equivalent cluster} \,\},$
and for $c_s,c_t\in\mathcal{V}$,
$
(c_s,c_t)\in\mathcal{E}
\iff
c_s \Rightarrow c_t \ \text{w.r.t.\ } x,
$
i.e., whenever any (equivalently, every) answer in $c_s$ is true, any (equivalently, every) answer in $c_t$ is also true, in all situations compatible with $x$.

\begin{proposition}\label{pro:entailment}
Let $x$ be a query and $L$ be an LLM. The implication relation on $\mathcal{V}$ of $\LAG(x\mid L)$ is reflexive and transitive.
\end{proposition}

These properties both establish ImG as a DAG (Theorem~\ref{thm:dag}) and provide verifiable criteria for whether a relation extraction method---e.g., a Natural Language Inference (NLI) model---produces predictions consistent with our assumptions (transitivity is checked empirically on sampled answer triples).

Implication is antisymmetric at the level of equivalence classes: if two clusters imply each other, they are equivalent and thus identified as one vertex. Suppressing self-loops yields the following.

\begin{theorem}\label{thm:dag}
For a given LLM $L$ and query $x$, the graph $\LAG(x\mid L)$ is a DAG.
\end{theorem}

Let $\mathcal{R}(\LAG(x\mid L))$ denote the set of roots of $\LAG(x\mid L)$. By Remark~\ref{rem:root}, $\mathcal{R}(\LAG(x\mid L))\neq\emptyset$, a necessary precondition for the aggregated distribution in IGE to be valid.

\paragraph{Incompatibility Graph.}
For a query $x$, two answers $a$ and $b$ are {compatible w.r.t.\ $x$} if there exists
at least one situation consistent with $x$ in which both $a$ and $b$ are
simultaneously true. Conversely, $a$ and $b$ are {incompatible
w.r.t.\ $x$} if no such situation exists. 
The \textbf{InG} is defined as \(
\CG(x\mid L)=(\mathcal{V},\mathcal{E}),
\mathcal{V}= \mathcal{R}(\LAG(x\mid L)).
\)
We encode incompatibility as a {symmetric relation} on roots; concretely, whenever two roots are incompatible, we include {both} directed edges $(c_s,c_t)$ and $(c_t,c_s)$ in $\mathcal{E}$. Formally,
$
\mathcal{E} = \Bigl\{ (c_s,c_t) \in \mathcal{V} \times \mathcal{V} \ \Bigm|\ 
c_s \neq c_t, \ \forall a \in c_s,\ \forall b \in c_t, 
\ a \text{ and } b \text{ are incompatible w.r.t.\ } x \Bigr\}.$
Since all answers within a cluster are logically equivalent by construction, incompatibility between clusters is well-defined (equivalently, it suffices to check any representatives).

\begin{proposition}\label{pro:incompatibility}
The incompatibility relation is symmetric, non-reflexive, and non-transitive.
\end{proposition}

Symmetry likewise provides an empirical sanity check: a reliable extractor should predict incompatibility consistently in both directions. Non-reflexivity and non-transitivity hold by definition and need no verification.

\paragraph{From Semantic to Logical Relations.} We operationalize the three relations with an off-the-shelf NLI model (DeBERTa-Large-MNLI) applied to each ordered answer pair. {Bidirectional} entailment ($a\Rightarrow b$ and $b\Rightarrow a$) is semantic equivalence and defines clusters exactly as SE does. The logical structure comes from the two signals SE discards: a {unidirectional} entailment ($a\Rightarrow b$, $b\not\Rightarrow a$) yields a directed implication edge in ImG, so a specific answer points to the more general ones it entails; a {contradiction} in either direction yields an incompatibility edge in InG. Compatible refinements are thus collapsed onto a shared root, while conflicting answers are flagged as incompatible. Since this relies on the NLI predictions exhibiting the transitivity (implication) and symmetry (incompatibility) of Propositions~\ref{pro:entailment} and~\ref{pro:incompatibility}, we validate both at scale in Appendix~\nameref{app:nli_verification}, finding a 97.0\% transitivity rate and 92.8\% symmetry.

\begin{figure*}[htbp]
    \centering
    \includegraphics[width=\linewidth]{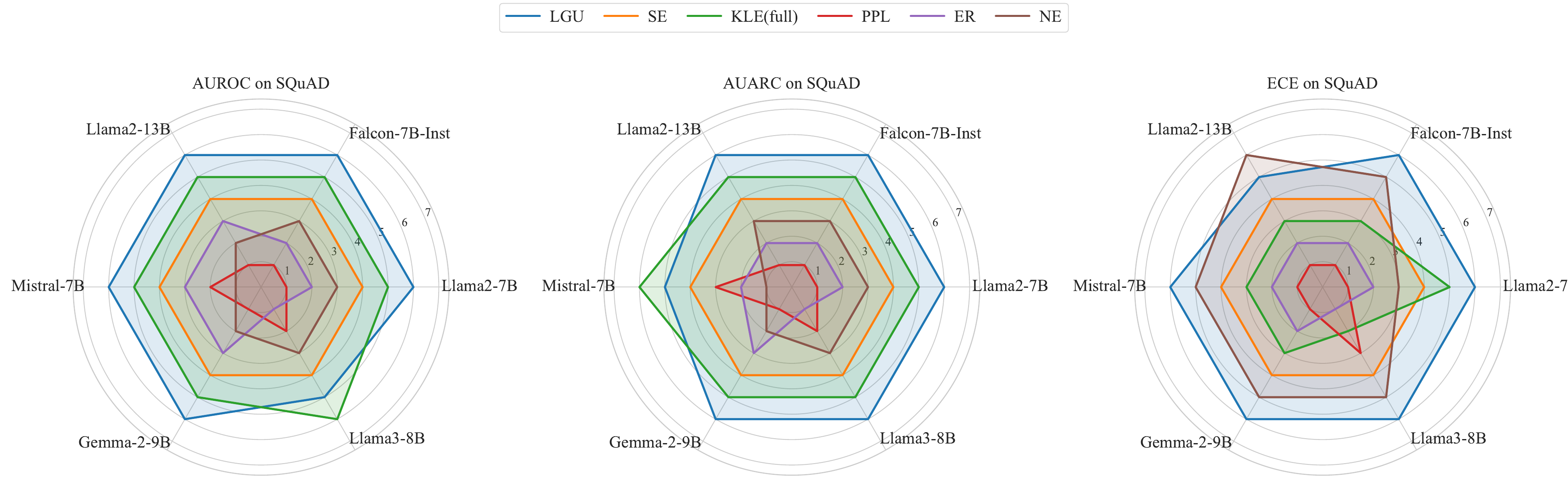}
    \caption{\textbf{Performance robustness across LLMs on SQuAD.} Radar charts show the rank-based performance of uncertainty methods across six model architectures. A larger shaded area corresponds to better and more consistent performance across tasks.}
    \label{fig:radar_chart}
\end{figure*}

\definecolor{lightpink}{RGB}{255,235,240}
\begin{table*}[t]
    \centering
    \caption{\textbf{Experimental Results for Falcon-7b-instruct.} We report AUROC ($\uparrow$), AUARC ($\uparrow$), and ECE ($\downarrow$). Best results for each metric are highlighted in \textbf{bold}.}
    \label{tab:falcon_results}
    \renewcommand{\arraystretch}{1.2}
    
    \resizebox{\textwidth}{!}{
        \begin{tabular}{c|l|ccc|ccc|ccc|ccc}
            \toprule
            
            & \multirow{2}{*}{\textbf{Method}} 
            & \multicolumn{3}{c|}{\textbf{Trivia QA}} 
            & \multicolumn{3}{c|}{\textbf{SQuAD}} 
            & \multicolumn{3}{c|}{\textbf{BioASQ}} 
            & \multicolumn{3}{c}{\textbf{NQ}} \\
            
            \cmidrule(lr){3-5} \cmidrule(lr){6-8} \cmidrule(lr){9-11} \cmidrule(lr){12-14}
            
            & 
            & \textbf{AUROC} & \textbf{AUARC} & \textbf{ECE} 
            & \textbf{AUROC} & \textbf{AUARC} & \textbf{ECE} 
            & \textbf{AUROC} & \textbf{AUARC} & \textbf{ECE} 
            & \textbf{AUROC} & \textbf{AUARC} & \textbf{ECE} \\
            \midrule
            
            \multirow{6}{*}{\rotatebox[origin=c]{90}{\shortstack{\textbf{Falcon-7b-}\\\textbf{instruct}}}} 
            
            & SE & \eb{0.824}{0.039} & \eb{0.687}{0.048} & \eb{0.087}{0.034} & \eb{0.667}{0.050} & \eb{0.544}{0.059} & \eb{0.201}{0.044} & \eb{0.678}{0.052} & \eb{0.596}{0.057} & \eb{0.184}{0.043} & \eb{0.681}{0.053} & \eb{0.500}{0.061} & \textbf{\eb{0.112}{0.039}} \\

            & NE & \eb{0.781}{0.044} & \eb{0.661}{0.051} & \eb{0.113}{0.047} & \eb{0.610}{0.054} & \eb{0.516}{0.058} & \eb{0.105}{0.042} & \eb{0.672}{0.053} & \eb{0.581}{0.060} & \textbf{\eb{0.091}{0.043}} & \eb{0.700}{0.049} & \eb{0.507}{0.063} & \eb{0.214}{0.054} \\
            
            & Embedding Regression & \eb{0.728}{0.051} & \eb{0.622}{0.054} & \eb{0.248}{0.044} & \eb{0.606}{0.055} & \eb{0.510}{0.058} & \eb{0.289}{0.046} & \eb{0.661}{0.054} & \eb{0.579}{0.056} & \eb{0.274}{0.044} & \eb{0.684}{0.053} & \eb{0.483}{0.061} & \eb{0.237}{0.043} \\

            & PPL & \eb{0.694}{0.052} & \eb{0.616}{0.056} & \eb{0.500}{0.051} & \eb{0.578}{0.056} & \eb{0.500}{0.062} & \eb{0.512}{0.054} & \eb{0.644}{0.052} & \eb{0.585}{0.055} & \eb{0.467}{0.051} & \eb{0.664}{0.058} & \eb{0.484}{0.062} & \eb{0.563}{0.048} \\
            
            & KLE(Full) & \eb{0.840}{0.039} & \eb{0.698}{0.046} & \textbf{\eb{0.073}{0.031}} & \eb{0.712}{0.051} & \eb{0.573}{0.057} & \eb{0.211}{0.045} & \eb{0.718}{0.048} & \eb{0.614}{0.055} & \eb{0.174}{0.044} & \eb{0.732}{0.050} & \eb{0.527}{0.060} & \eb{0.147}{0.043} \\
            \cmidrule{2-14}
            
            & \colorbox{lightpink}{\textbf{LGU (Ours)}} & \colorbox{lightpink}{\textbf{\eb{0.849}{0.038}}} & \colorbox{lightpink}{\textbf{\eb{0.702}{0.046}}} & \colorbox{lightpink}{\eb{0.151}{0.040}} & \colorbox{lightpink}{\textbf{\eb{0.714}{0.050}}} & \colorbox{lightpink}{\textbf{\eb{0.579}{0.059}}} & \colorbox{lightpink}{\textbf{\eb{0.098}{0.040}}} & \colorbox{lightpink}{\textbf{\eb{0.726}{0.048}}} & \colorbox{lightpink}{\textbf{\eb{0.619}{0.056}}} & \colorbox{lightpink}{\eb{0.132}{0.041}} & \colorbox{lightpink}{\textbf{\eb{0.752}{0.049}}} & \colorbox{lightpink}{\textbf{\eb{0.533}{0.061}}} & \colorbox{lightpink}{\eb{0.138}{0.038}} \\
            
            \bottomrule
        \end{tabular}
    }
\end{table*}

\subsection{Graph-Based Uncertainty}

Let $C=\{c_1,\dots,c_t\}$ be the equivalent clusters induced by $L$ for query $x$, with probabilities $P=\{p_1,\dots,p_t\}$.

\paragraph{Implication Graph Entropy.}
Let $\mathcal{R}(\LAG(x\mid L))=\{r_1,\dots,r_k\}$. For each root $r_i$, define the aggregated mass
$p'(r_i)= \sum_{c_j\in \mathrm{Reach}(r_i)} p_j,$
where $\mathrm{Reach}(r_i)$ is the set of vertices reachable from $r_i$ (including $r_i$). Since reachability sets may overlap, we normalize $p(r_i)= \frac{p'(r_i)}{\sum_{j=1}^{k} p'(r_j)}.$
The \textbf{IGE} is
$$\EE(x\mid L)= -\sum_{i=1}^{k} p(r_i)\log p(r_i).$$

\paragraph{Probabilistic Interpretation.} The aggregation $p'(r_i)=\sum_{c_j\in\mathrm{Reach}(r_i)} p_j$ marginalizes over a logically coherent hypothesis space: since a root $r_i$ entails every cluster reachable from it, assigning the mass of a general cluster $c_j$ to each root that entails it attributes that mass to the specific hypotheses {consistent with} $c_j$ rather than treating $c_j$ as a mutually exclusive outcome. Thus $p'(r_i)$ is the evidential support the sample provides for the strongest hypothesis $r_i$. As overlapping reachability sets give $\sum_i p'(r_i)\ge 1$, the normalization $p(r_i)=p'(r_i)/\sum_j p'(r_j)$ restores a proper distribution over the roots. IGE is the Shannon entropy of this distribution: low when mass concentrates on a single logically strongest hypothesis and high when dispersed across genuinely distinct ones, independent of their mutual incompatibility (which LGU handles separately).

\paragraph{Discrete IGE (DIGE).}
In black-box settings, we sample $n$ outputs $a_1,\dots,a_n$ and estimate cluster probabilities by empirical frequencies; $\mathtt{DIGE}(x\mid L)$ is computed by substituting these estimates into the definition of $\EE(x\mid L)$.

\paragraph{Incompatibility Score.}
Let $\CG(x\mid L)=(\mathcal{V},\mathcal{E})$ be the InG, a directed graph with no self-loops.
Depending on the data type, $\mathcal{E}$ can be (i) binary, recording whether two root hypotheses are incompatible,
or (ii) weighted, equipped with a weight function $w:\mathcal{E}\to[0,1]$ where each edge $(u,v)\in\mathcal{E}$ carries a normalized weight $w_{uv}$ that estimates the probability that roots $u$ and $v$ are incompatible. Let $s(u,v)\in[0,1]$ be the incompatibility probability assigned to an ordered pair of answers by the same relation-extraction module used to build the graph. For $u\in \mathcal{V}$ and $v\in \mathcal{V}$ we set
\begin{equation*}
w_{uv}=\tfrac{1}{2}\bigl(s(u,v)+s(v,u)\bigr)\in[0,1].
\end{equation*}
In our experiments this module is an off-the-shelf NLI model, for which $s(u,v)$ is the softmax mass on the {contradiction} label. We view the InS as a family of {graph-density metrics} that summarize how pervasive mutual inconsistency is among the roots, and we ablate several plausible instantiations:
\begin{itemize}[leftmargin=*, topsep=0pt, itemsep=2pt]

\item \textbf{(Weighted) Edge Density.}
A normalized sparsity measure in network analysis~\cite{Newman2010,Fortunato2010}. Since a loop-free directed graph on $|\mathcal{V}|$ nodes has at most $|\mathcal{V}|(|\mathcal{V}|-1)$ edges, we define the directed edge density
\[\CS_{\mathrm{dens}}(x\mid L)=
\begin{cases}
\displaystyle \frac{|\mathcal{E}|}{|\mathcal{V}|(|\mathcal{V}|-1)}, & |\mathcal{V}|\ge 2,\\[6pt]
0, & |\mathcal{V}|\le 1.
\end{cases}\]
This yields $\CS_{\mathrm{dens}}(x\mid L)\in[0,1]$ and is directly comparable across queries with different numbers of roots. When edge weights $w_{uv}\in[0,1]$ are available, we replace the edge count with total weight.

\item \textbf{Average In-/Out-Degree.}
A scale-aware statistic capturing how broadly incompatibility is distributed. Define the average in- and out-degree over roots as

$\CS_{\mathrm{in}}(x\mid L) = \frac{1}{|\mathcal{V}|}\sum_{u\in\mathcal{V}} d^{-}(u), $ and $
\CS_{\mathrm{out}}(x\mid L) = \frac{1}{|\mathcal{V}|}\sum_{u\in\mathcal{V}} d^{+}(u).$
For directed graphs, $\sum_{u} d^{-}(u)=\sum_{u} d^{+}(u)=|\mathcal{E}|$, hence
$\CS_{\mathrm{in}}(x\mid L)=\CS_{\mathrm{out}}(x\mid L)=|\mathcal{E}|/|\mathcal{V}|$.

\item \textbf{Estrada Index.}
To capture global (spectral) connectivity/complexity beyond local edges, we consider the Estrada index of the incompatibility graph, a walk-based spectral functional linked to matrix-exponential centrality and communicability~\cite{PenaGutmanRada2007,GutmanDengRadenkovic2011,ESTRADA201289}.
Let $A$ be the adjacency matrix of $\CG(x\mid L)$. Let $\{\lambda_i\}_{i=1}^{|\mathcal{V}|}$ be the eigenvalues of $A$ (the adjacency spectrum). The Estrada index is defined as
$
\mathtt{EE}(\CG(x\mid L))= \sum_{i=1}^{|\mathcal{V}|} e^{\lambda_i}.
$ We use
$
\CS_{\mathrm{EE}}(x\mid L)= \mathtt{EE}(\CG(x\mid L)).$

\end{itemize}

\begin{remark}
    Although we represent {InG} as a directed graph, incompatibility is a symmetric relation: for any $(u,v)\in\mathcal{E}$ we also have $(v,u)\in\mathcal{E}$ (and also in the weighted case, $w_{uv}=w_{vu}$). Therefore, $A$ is a real symmetric matrix with $A_{uv}=A_{vu}$ (taking value $1$ for unweighted edges and $w_{uv}\in[0,1]$ when weights are available), so its eigenvalues are all real.
\end{remark}

\paragraph{Logical Graph Uncertainty.}
We define \textbf{LGU} as $$\SGE(x\mid L)= \bigl(1+\CS(x\mid L)\bigr)\,\EE(x\mid L).$$

\paragraph{Design Rationale.} The multiplicative form $(1+\CS)\,\EE$ is the simplest operator satisfying four desiderata: (D1)~{reduction} to pure implication entropy when $\CS=0$ (guaranteed by the additive ``$1+$''); (D2)~{monotonicity} in $\CS\ge0$; (D3)~{scale coupling}, so that incompatibility amplifies uncertainty only in proportion to the dispersed mass---a confident answer ($\EE=0$) stays certain ($\CS\cdot 0=0$), which an additive form $\EE+\CS$ violates both behaviourally and dimensionally (it sums a nats-valued entropy with a dimensionless density); and (D4)~{boundedness}, since Theorem~\ref{thm:EEextreme} gives the structure-determined maximum $(1+\CS)\log|\mathcal{V}|$ needed for calibration. The dimensionless factor $(1+\CS)$ rescales the entropy $\EE$ directly, so $(1+\CS)\EE$ remains an entropy while staying parameter-free; the ablation in Appendix~\nameref{sec:ablation_study} confirms empirically that including the incompatibility factor improves over the IGE-only variant ($\CS=0$). A full statement and discussion of D1--D4 is given in Appendix~\nameref{app:design_rationale}.

LGU modulates implication-based uncertainty by incompatibility among the most specific (root) hypotheses. The following results hold for the instantiations of $\CS$ considered in this paper (i.e., edge density, weighted density, average degree, and Estrada index).

\begin{theorem}[Extremal Values]\label{thm:EEextreme}
Let $L$ be an LLM and $x$ a query. Then:
\begin{itemize}[leftmargin=*, topsep=0pt, itemsep=2pt]
\item $\SGE(x\mid L)\le (1+\CS(x\mid L))\log |\mathcal{V}(\LAG(x\mid L))|$, and equality holds if and only if the distribution over equivalent clusters is uniform, $|\mathcal{R}(\LAG(x\mid L))|=|\mathcal{V}(\LAG(x\mid L))|$ (i.e., the ImG is empty), and the InG is complete;
\item $\SGE(x\mid L)\ge 0$, and equality holds if and only if $|\mathcal{R}(\LAG(x\mid L))|=1$.
\end{itemize}
\end{theorem}

Theorem~\ref{thm:EEextreme} characterizes the attainable range of $\SGE(x\mid L)$ graph-theoretically. The upper bound makes explicit the multiplicative contribution of $\CS$ and is attained in the maximally uncertain regime: a uniform distribution over equivalence classes, an empty implication subgraph ($|\mathcal{R}|=|\mathcal{V}|$), and a complete incompatibility subgraph on the roots; the lower bound $\SGE(x\mid L)\ge 0$ holds exactly when the root set is a singleton. Since the maximum $(1+\CS)\log|\mathcal{V}|$ is fixed by the graph structure given $\CS$, this bounded range lets scores be normalized in a principled way for ECE evaluation. Full proof is in Appendix~\nameref{app:proof_EEextreme}.

We next analyze the computational cost of computing $\SGE(x\mid L)$ from $n$ sampled candidate answers.

\begin{theorem}[Time Complexity]\label{thm:lgu_complexity}
Let the algorithm take $n$ candidate answers as input. Assume:
(i) equivalence, implication, and incompatibility primitives are implemented via pairwise checks,
and (ii) the chosen InS $\CS(\cdot)$ can be computed on any directed graph with
$k$ vertices in $\mathcal{O}(k^2)$ time.
Then computing $\SGE(x\mid L)=(1+\CS(x\mid L))\,\EE(x\mid L)$ takes $\mathcal{O}(n^2)$ time in the worst case,
and this bound is tight.
\end{theorem}

\section{Experiments}
\label{sec:exp}

We systematically validate our method’s LLM generation uncertainty estimation performance across diverse QA benchmarks and mainstream LLMs. All details of the following experimental section can be found in Appendix~\nameref{app:experiments_details}.

\paragraph{Datasets.} To comprehensively evaluate our method, we conduct experiments across four question-answering datasets spanning three diverse categories. For General Knowledge, we utilize \textbf{TriviaQA} \cite{triviaqa} and \textbf{SQuAD} \cite{squad}; for Specialized Domains, we employ \textbf{BioASQ} \cite{bioasq} within the biomedical field; and for search queries, we use NQ \cite{nq} to represent real-world user search behavior.

\begin{table}[t]
    \centering
    \caption{\textbf{Logical Question Answer Rate.} Non-empty ImG suitable for LGU.}
    \label{tab:logical_rate}
    \renewcommand{\arraystretch}{1.1}
    \scalebox{0.75}{
        \begin{tabular}{@{\,}lccccc@{\,}}
        \toprule
        \textbf{Model} & \textbf{Trivia QA} & \textbf{SQuAD} & \textbf{BioASQ} & \textbf{NQ} \\
        \midrule
        Llama-2-7b-chat    & 0.483 & 0.890 & 0.773 & 0.773 \\
        Falcon-7b-instruct & 0.903 & 0.990 & 0.965 & 0.988 \\
        Falcon-7b          & 0.873 & 0.993 & 0.960 & 0.985 \\
        Mistral-7B-v0.1    & 0.750 & 0.980 & 0.923 & 0.975 \\
        Llama-2-13b-chat   & 0.393 & 0.833 & 0.708 & 0.730 \\
        Qwen3-32B          & 0.623 & 0.930 & 0.753 & 0.885 \\
        Llama-3-8B         & 0.578 & 0.955 & 0.850 & 0.953 \\
        gemma-2-9b         & 0.523 & 0.858 & 0.620 & 0.780 \\
        \cmidrule{2-5}
        Average Rate        & 0.641 & 0.929 & 0.807 & 0.884 \\
        \bottomrule
    \end{tabular}}
\end{table}

\paragraph{Models.} We evaluate a range of modern language models, including \textbf{Llama} \cite{llama2}, \textbf{Falcon} \cite{falcon}, \textbf{Mistral} \cite{mistral}, \textbf{Qwen} \cite{qwen3}, and \textbf{Gemma} \cite{gemma2}. Additionally, we use \textbf{DeBERTa-Large-MNLI} \cite{deberta} as our NLI model to construct logical graphs and semantic clusters. To assess the correctness of model generations, we employ GPT-3.5 \cite{gpt3.5} as the reference judge.

\begin{figure*}[htbp]
    \centering

    \begin{subfigure}[t]{0.48\linewidth}
        \centering
        \includegraphics[width=\linewidth]{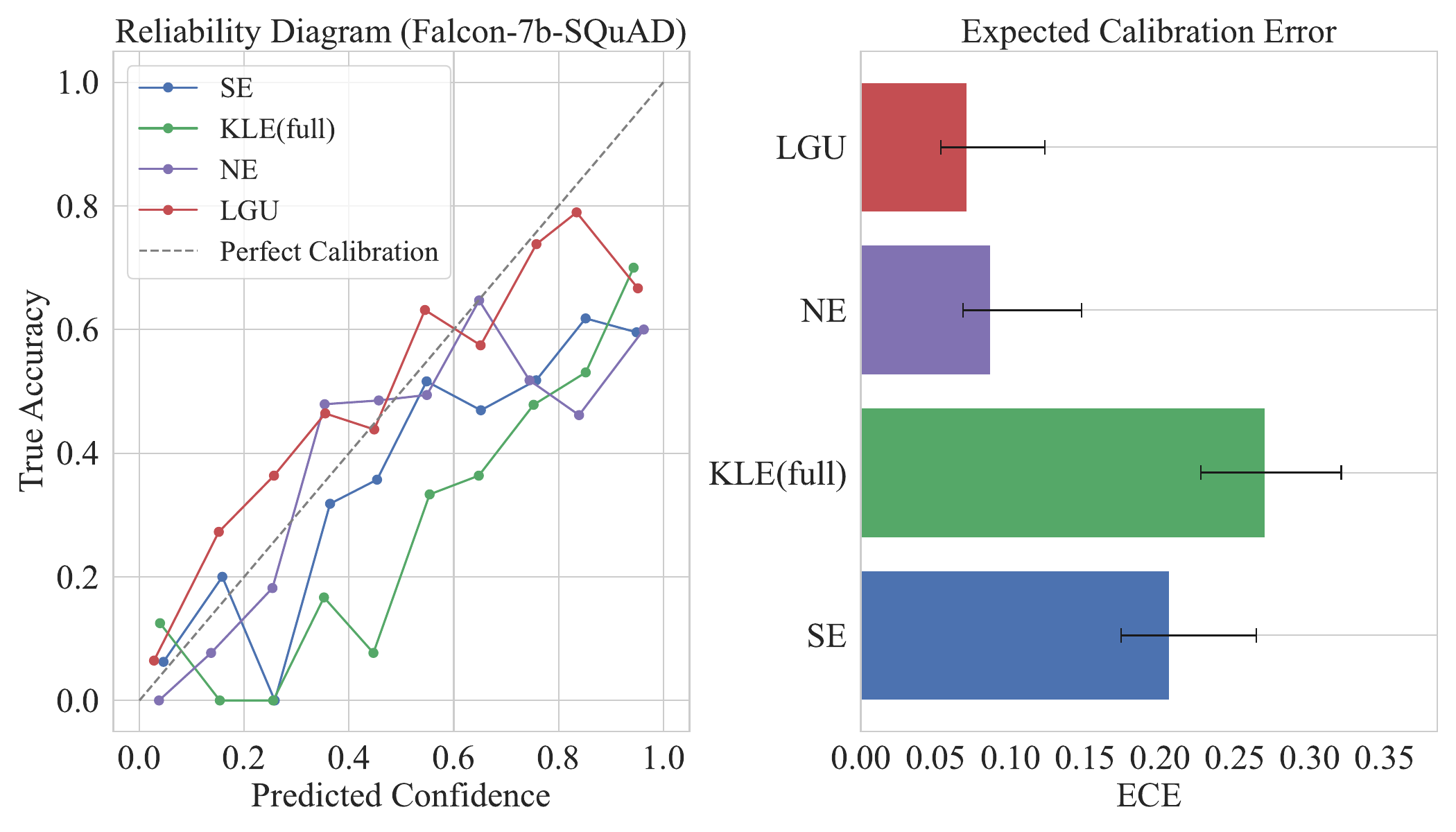}
        \caption{Calibration Analysis on Falcon-7B (SQuAD).}
        \label{fig:falcon7b_squad_calibration}
    \end{subfigure}
    \hfill
    \begin{subfigure}[t]{0.48\linewidth}
        \centering
        \includegraphics[width=\linewidth]{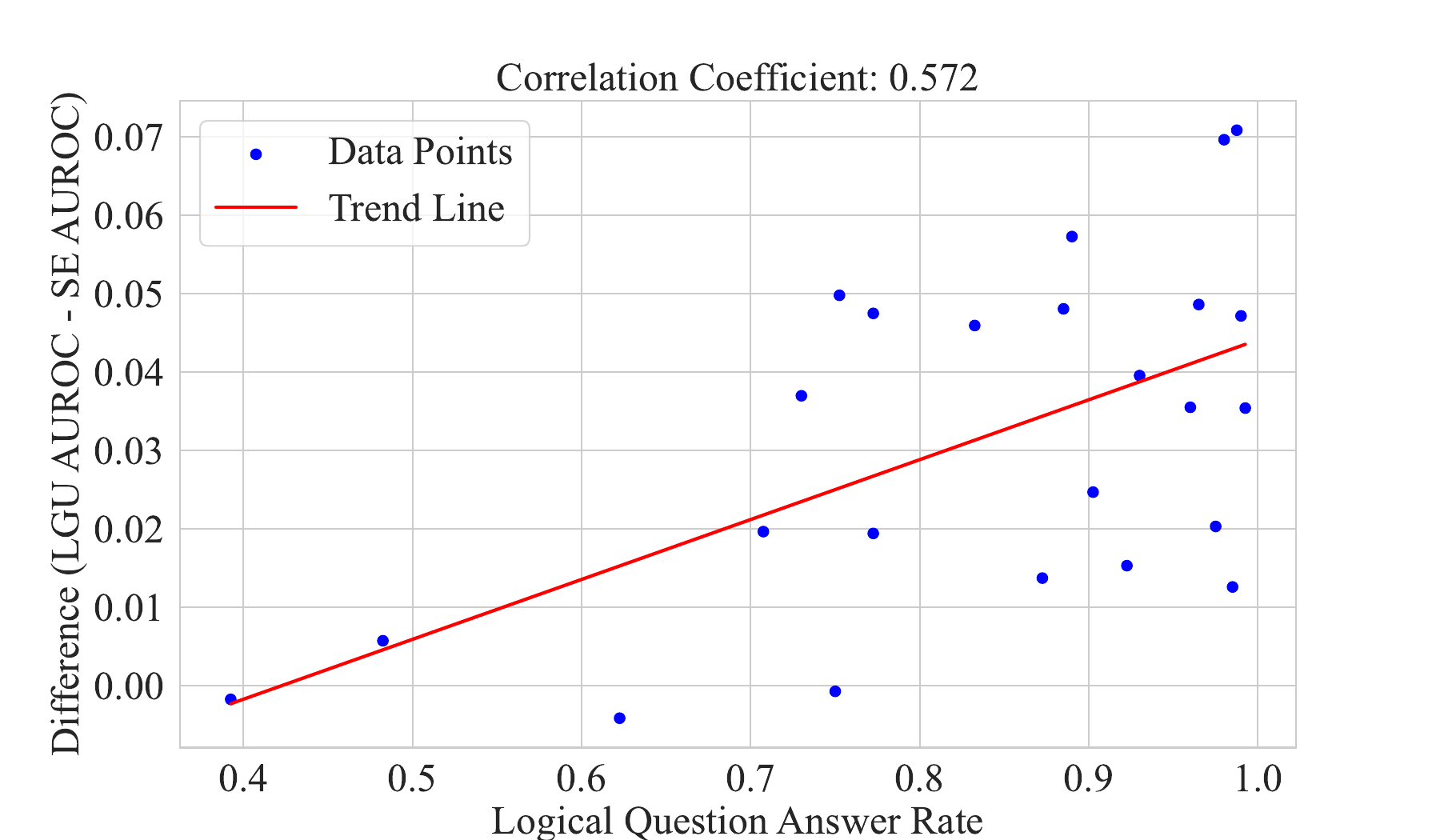}
        \caption{Correlation Between Logical Structure Prevalence and LGU Gain.}
        \label{fig:lgu_se_correlation}
    \end{subfigure}

    \caption{
    \textbf{Effectiveness and Structural Sensitivity of LGU.}
    \textbf{(a) Calibration of Falcon-7b on SQuAD.} \textit{Left:} Reliability diagrams aligning predicted confidence with empirical accuracy. \textit{Right:} Expected Calibration Error (ECE, $\downarrow$). \textbf{(b) LGU performance gain vs.\ logical structure prevalence.} Correlation between the AUROC improvement over SE ($\Delta$AUROC) and the proportion of answers containing logical dependencies (non-empty ImG).
    }
    \label{fig:lgu_combined}
\end{figure*}

\paragraph{Baselines.}
We evaluate against three baseline categories: \textit{semantic methods}, including Semantic Entropy \cite{se,se_nature} and Kernel Language Entropy (KLE) \cite{kle}; \textit{token-level metrics}, comprising Maximum Sequence Probability (MSP), Perplexity (PPL) \cite{perplexity}, and \(\text{Avg}/\text{Max}(-\log p)\) \cite{selfcheck}; and a \textit{supervised} Embedding Regression baseline \cite{sequence1}. Please refer to Appendix~\nameref{app:experiments_details} for exact baseline formulations and implementation details.

\paragraph{Evaluation Metrics.}
We evaluate performance using three standard metrics: \textit{discrimination} via AUROC \cite{se_nature}, {practical utility} for selective prediction via AUARC \cite{auarc}, and {calibration} reliability via ECE \cite{ece}. 
Appendix~\nameref{app:metric} shows our calibration evaluation is robust to the normalization protocol, with LGU the most frequently best-calibrated method under both theoretical and empirical bounds.


\subsection{Results}



Our empirical evaluation demonstrates that LGU offers strong discriminative power, robust cross-model generalization, and reliable calibration. To characterize this honestly, across the $96$ model$\times$dataset$\times$metric combinations LGU attains the best {absolute} value in roughly half, while ranking first on {average} across all settings---the most consistent performer overall. All detailed results are provided in Appendix~\nameref{app:generate}.

\paragraph{Discrimination and Calibration Results.} As detailed in Table~\ref{tab:falcon_results}, on the Falcon-7b-instruct model LGU attains the strongest average rank in both discrimination and calibration, achieving the best AUROC and AUARC across the four datasets and outperforming strong semantic-level baselines such as SE and KLE (e.g., on NQ it reaches AUROC $0.752$ and AUARC $0.533$); on TriviaQA its margin over KLE(Full) and SE is small and within noise.

\paragraph{Cross-Model Robustness and Stability.} We evaluate stability across six diverse LLM families. Figure~\ref{fig:radar_chart} visualizes the rank-based performance of each method on SQuAD via radar charts, where a larger shaded area corresponds to a higher and more consistent relative ranking. LGU (blue area) encompasses the outermost perimeter in both the AUROC and AUARC plots, confirming robust discriminative power regardless of model scale (7B--13B) or architecture (Llama, Mistral, Gemma, Falcon).

\paragraph{Superior Calibration of LGU.} Figure~\ref{fig:falcon7b_squad_calibration} uses Falcon-7b on SQuAD as a representative example. The reliability diagram shows LGU (red curve) aligning most closely with the perfect-calibration line, and quantitatively LGU achieves the lowest ECE ($0.071$), outperforming NE ($0.087$), KLE ($0.270$), and SE ($0.206$)---indicating better-calibrated estimates in the settings where its advantage is statistically clear (non-overlapping reliability curves and error bars).

\subsection{Ablation Study}
\paragraph{Prevalence of Logical Structure Amplifies LGU's Advantage.} To investigate the source of LGU's improvement over SE, we analyze the relationship between the {Non-empty ImG Rate} (Table~\ref{tab:logical_rate}) and the performance gain across data subsets. The correlation analysis in Figure~\ref{fig:lgu_se_correlation} plots the proportion of samples whose ImG contains implication edges (x-axis) against the AUROC improvement of LGU over SE (y-axis): the upward regression trend with a Pearson coefficient of $0.572$ indicates a moderate-to-strong positive correlation. Where the construction rate is low ($<$0.5) the gap is negligible, whereas at high rates ($>$0.7) LGU's advantage becomes pronounced, with AUROC gains up to $0.07$---confirming that LGU capitalizes on structural information SE ignores.

\paragraph{Scope: LGU is not restricted to ``logical'' questions.} A non-empty ImG is {not} a precondition for applying LGU. The relevant distinction is not ``logical'' versus ``non-logical'' questions, but whether a query admits {a single atomic correct answer} or {multiple correct answers standing in logical relations} (entailment or incompatibility); the latter case is pervasive among open-ended or under-specified queries, which is precisely where semantic-only measures conflate benign diversity with uncertainty. When a query has a unique atomic answer and the samples are identical or logically unrelated, the ImG carries no edges and LGU reduces {exactly} to SE. LGU therefore never applies to a narrower set of inputs than SE---it strictly adds signal when logical structure is present and is otherwise identical---so the rates in Table~\ref{tab:logical_rate} quantify {how often this extra signal is available}, not where the method is valid.
\paragraph{Role of Incompatibility in Uncertainty Quantification.} To rigorously validate the architectural choices of our framework, we provide a comprehensive ablation study in Appendix~\nameref{sec:ablation_study}. By comparing LGU against reduced baselines (e.g., IGE), we demonstrate that modeling implication alone is insufficient. The study confirms that constructing the InG with an Edge Density score provides the necessary signal to identify conflicting hallucinations, empirically justifying our default configuration.



\section{Related Work}

Early UQ methods derive sequence-level confidence from token probabilities
via metrics such as average log-probability, perplexity, or token entropy
\cite{sequence1,xiao2021,fadeeva2024}, but operate at the lexical level,
conflating surface variation with genuine uncertainty, with
calibration degrading under RLHF fine-tuning \cite{prompt2}. Semantic-level
methods instead cluster answers into meaning-equivalent groups: SE
\cite{se,se_nature} computes entropy over NLI-derived equivalence classes,
and KLE \cite{kle} extends this via kernels at the cost of extra
hyperparameters; others add logit-based scoring \cite{semantic-energy} or
self-evaluation prompts \cite{prompt2}, though the latter tend to be
overconfident and underperform token-based methods \cite{ni2024}. Beyond
flat clustering, structured relational representations have recently entered
LLM UQ: \cite{jiang2024} build a response--claim bipartite
graph and score uncertainty via centrality (degree, PageRank), while
\cite{zhao2026} define structural entropy over directed semantic dependency
graphs; for long-form outputs, LUQ \cite{luq2024} aggregates claim-level
entailment consistency and LoGU \cite{logu2025} distinguishes supported,
unsupported, and contradicted claims. Unlike all of these, LGU explicitly
models truth-conditional implication and mutual incompatibility among
hypotheses, moving beyond flat semantic clustering toward a structured
logical hypothesis space.

\section{Discussion}
\label{sec:discussion}

This work addresses a fundamental limitation in current uncertainty quantification for large language models: the inability to capture the logical structure underlying model responses. We propose a logical-graph-based framework centered on two complementary quantities---IGE, which aggregates probability mass along entailment chains in an ImG, and LGU, which further penalizes dense mutual incompatibility among the roots via an InG. Across multiple benchmarks and model families, LGU substantially improves identification of unreliable outputs---especially in open-ended settings where correct answers are logically rich rather than trivially repetitive---promoting logical coherence to a first-class role alongside semantic equivalence.

\paragraph{Limitation.} First, the framework supports only implication and pairwise incompatibility, not richer relations such as causality, temporal order, or defeasible inference. Second, we operationalize these relations with an off-the-shelf NLI model---a proxy for {textual} rather than logical entailment---whose errors propagate to the graphs; a purpose-built extractor would benefit LGU, though our gains require no logical inference, stemming only from the {directional} entailment and {contradiction} signals the NLI model already produces but SE discards. Third, our evaluation targets short-form QA; extending LGU to long-form generation raises {alignment} and {probability-assignment} obstacles (detailed in Appendix~\nameref{app:longform}) that also make graph construction costly, motivating hierarchical or approximate inference. Finally, LGU's gains concentrate on logically structured queries and otherwise coincide with SE (Table~\ref{tab:logical_rate}), adding signal without ever underperforming SE.
\bibliography{LGU}


\clearpage
\appendix
\section{Proofs}

\subsection{Proof of Theorem~\ref{thm:EEextreme}}
\label{app:proof_EEextreme}

\begin{proof}
We prove the upper and lower bounds separately, and then characterize the equality cases.

{Upper bound.}
By definition,
\[
\SGE(x\mid L) = \bigl(1+\CS(x\mid L)\bigr)\,\EE(x\mid L).
\]
Since $\EE(x\mid L)$ is the Shannon entropy of a distribution on the roots
$\mathcal{R}(\LAG(x\mid L))=\{r_1,\dots,r_k\}$, we have
\[
\EE(x\mid L)\le \log k \le \log |\mathcal{V}(\LAG(x\mid L))|.
\]
Therefore,
\[
\SGE(x\mid L) \le \bigl(1+\CS(x\mid L)\bigr)\log |\mathcal{V}(\LAG(x\mid L))|.
\]

To characterize when equality holds, note that $\SGE(x\mid L)$ equals the right-hand side
if and only if $\EE(x\mid L)=\log|\mathcal{V}(\LAG(x\mid L))|$.
This requires both (i) $k=|\mathcal{V}(\LAG(x\mid L))|$ and (ii) the root distribution $\{p(r_i)\}_{i=1}^k$
to be uniform. Condition (i) is equivalent to $\mathcal{R}(\LAG(x\mid L))=\mathcal{V}(\LAG(x\mid L))$,
i.e., every vertex is a root and the ImG has no edges.
In this case, $\mathrm{Reach}(r_i)=\{r_i\}$ for all $i$, so $p(r_i)=p_i$ and
$\EE(x\mid L)$ reduces to the entropy of the original cluster distribution; hence (ii) holds
if and only if the distribution over equivalent clusters is uniform.
Finally, since all instantiations of $\CS$ considered in this paper achieve their
maximum value exactly when InG is complete, the stated equality condition follows.

{Lower bound.}
Since $\CS(x\mid L)\ge 0$ by definition of an InS and $\EE(x\mid L)\ge 0$ for Shannon entropy,
it follows immediately that
\[
\SGE(x\mid L) = (1+\CS(x\mid L))\,\EE(x\mid L)\ge 0.
\]
If $|\mathcal{R}(\LAG(x\mid L))|=1$, then the induced distribution on roots is a point mass and
$\EE(x\mid L)=0$, which implies $\SGE(x\mid L)=0$.

Conversely, suppose $\SGE(x\mid L)=0$. As $1+\CS(x\mid L)>0$, we must have $\EE(x\mid L)=0$.
Under the standard full-support assumption on clusters (each induced cluster has positive probability),
every root $r_i$ has $p(r_i)>0$, hence $\EE(x\mid L)=0$ can occur only when $k=|\mathcal{R}(\LAG(x\mid L))|=1$.
This completes the proof.
\end{proof}

\subsection{Proof of Theorem~\ref{thm:lgu_complexity}}\label{app:proof_lgu_complexity}
\begin{proof}
We bound each stage and then combine the results.

{Equivalent Clustering.}
Forming equivalent clusters from $n$ sampled answers via pairwise equivalence checks examines
$\Theta(n^2)$ pairs in the worst case, hence takes $\Theta(n^2)$ time.

{ImG Construction and IGE Aggregation.}
Let $t$ be the number of clusters. In the worst case, $t=\Theta(n)$ (each answer forms its own cluster).
Constructing the ImG $\LAG(x\mid L)$ by pairwise implication checks over clusters costs
$\Theta(t^2)\subseteq \mathcal{O}(n^2)$ time.
Once $\LAG(x\mid L)$ is built, aggregating probabilities and computing $\EE(x\mid L)$ can be done by a
topological traversal that processes each vertex and edge once, i.e., $\mathcal{O}(t+m)$ time,
where $m=|\mathcal{E}(\LAG(x\mid L))|$.
Since $\LAG(x\mid L)$ is a DAG, $m\le t(t-1)/2=\Theta(t^2)$, hence $\mathcal{O}(t+m)\subseteq \mathcal{O}(t^2)
\subseteq \mathcal{O}(n^2)$.

{InG Construction and InS.}
Let $k=|\mathcal{R}(\LAG(x\mid L))|$ be the number of roots; then $1\le k\le t\le n$.
Constructing the incompatible graph $\CG(x\mid L)$ by pairwise incompatibility checks over roots costs
$\Theta(k^2)\subseteq \mathcal{O}(n^2)$ time.
By assumption, the chosen score $\CS(x\mid L)$ is computable on a $k$-vertex digraph in $\mathcal{O}(k^2)$ time,
which is again bounded by $\mathcal{O}(n^2)$.

The final scalar combination $\SGE(x\mid L)=(1+\CS(x\mid L))\,\EE(x\mid L)$ takes $\mathcal{O}(1)$ time. Summing the stages, the overall running time is mainly dominated by the quadratic pairwise checks, yielding
$\mathcal{O}(n^2)$ in the worst case.

{Tightness.}
In the worst case, $\Omega(n^2)$ pairwise relations must be examined: e.g., when no two answers are
equivalent, determining the clustering requires inspecting \(\Theta(n^2)\) pairs. Therefore the
\(\mathcal{O}(n^2)\) is tight.
\end{proof}

\section{Design Rationale for the Multiplicative Form}
\label{app:design_rationale}

Here we give the full statement of the four desiderata that single out the multiplicative operator $(1+\CS)\,\EE$ among elementary ways of combining an implication-based entropy $\EE$ with an InS $\CS$.

\begin{itemize}[leftmargin=*, topsep=2pt, itemsep=3pt]
\item \textbf{(D1) Reduction.} When the roots carry no mutual incompatibility ($\CS=0$), LGU must reduce to the pure implication-based entropy, $\SGE=\EE$. The additive ``$1+$'' guarantees this rather than zeroing the score out.
\item \textbf{(D2) Monotonicity.} For fixed dispersion $\EE$, denser incompatibility among roots should never decrease uncertainty. Since $\CS\ge 0$ is nondecreasing in the incompatibility edge set for every instantiation we use, $(1+\CS)\EE$ is monotone in $\CS$.
\item \textbf{(D3) Scale coupling.} Incompatibility should amplify uncertainty only in proportion to how much mass is actually dispersed---if the model commits to a single root ($\EE=0$), no amount of incompatibility should manufacture uncertainty. A multiplicative coupling enforces this ($\CS\cdot 0=0$), whereas an additive form $\EE+\CS$ would spuriously report uncertainty for a confident answer. The additive form is moreover {dimensionally inconsistent}: $\EE$ is a Shannon entropy carrying units of nats (or bits), while $\CS$ is a dimensionless graph-density statistic, so the bare sum $\EE+\CS$ adds incommensurable quantities. The factor $(1+\CS)$ is a constant: being dimensionless, it rescales the entropy $\EE$ directly, so $(1+\CS)\EE$ remains an entropy while staying principled and parameter-free.
\item \textbf{(D4) Boundedness/calibratability.} The score must admit a structure-determined maximum so it can be normalized for calibration. Theorem~\ref{thm:EEextreme} shows $(1+\CS)\,\EE$ is bounded by $(1+\CS)\log|\mathcal{V}|$ with explicit extremal configurations, which additive or higher-order couplings do not yield as cleanly.
\end{itemize}

These four properties single out $(1+\CS)\,\EE$ among elementary combinations: the additive coupling $\EE+\CS$ violates (D3) both behaviourally (manufacturing uncertainty when $\EE=0$) and dimensionally (summing incommensurable units unless a dimensioned hyperparameter is reintroduced), while the ablation in Appendix~\nameref{sec:ablation_study} confirms empirically that including the incompatibility factor improves over the IGE-only variant ($\CS=0$).

\section{Extending LGU to Long-Form Generation}
\label{app:longform}

Our evaluation targets short-form QA, where each answer is a single atomic proposition with a directly available probability. Extending LGU to long-form or multi-sentence generation raises two obstacles. The first is {alignment}: a long-form response must be decomposed into constituent claims and those claims matched across responses, a correspondence that is neither one-to-one nor context-free, since a single claim may be split, merged, or restated across generations. The second is {probability assignment}: token-derived sequence probabilities are sensitive to length and phrasing, so no stable probability mass attaches to a multi-sentence hypothesis, undermining the mass-aggregation step that IGE relies on. Both obstacles also make graph construction substantially more costly as the number of claims grows, motivating hierarchical decomposition or approximate inference over the implication and incompatibility graphs. We leave these extensions to future work.

\section{Experiment Details and Ablations}
\label{app:exp}

\subsection{Experiments Details}
\label{app:experiments_details}

\begin{table*}[ht]
\caption{\textbf{Prompt Templates for Evaluating Answer Correctness and Semantic Entailment.} Placeholders like \texttt{<Question>} are replaced with specific instances during experiments.
}
\label{tab:prompt_templates}
\centering
\begin{tabularx}{\linewidth}{llX}
\toprule
\textbf{No.} & \textbf{Task / Objective} & \textbf{Prompt Template} \\
\midrule
1 & \textbf{Evaluating Correctness} & 
\texttt{We are evaluating answers to the question: <Question>}\newline
\texttt{Here are two possible answers:}\newline
\texttt{Possible Answer 1: <Predicted Answer>}\newline
\texttt{Possible Answer 2 set: <Correct Answer(s)>}\newline
\texttt{\rule{0pt}{1.2em}} 
\texttt{Within the context of the question, does the Predicted Answer express the same meaning as any answer in the Reference Answer Set (even if worded differently)? Consider paraphrasing, synonyms, and equivalent expressions as correct. If they mean the same, respond with yes. If they are different in meaning, respond with no.}\newline
\texttt{Response:}
\\ \midrule
2 & \textbf{Brief Answer Generation} & 
\texttt{Answer the following question in a single brief but complete sentence.}\newline
\texttt{Question: <Question>}\newline
\texttt{Answer:}
\\
\bottomrule
\end{tabularx}
\end{table*}

\paragraph{Implementation Details.}
Our study leverages the open-source models suite alongside the DeBERTa-Large-MNLI model, both accessed through the HuggingFace Transformers framework to ensure reproducibility. Our experimental setup consists of a single server equipped with four NVIDIA RTX 4090 GPUs. All models were implemented in PyTorch 2.7 and trained using CUDA 12.2. We have made the core code publicly available. Please refer to \url{https://anonymous.4open.science/r/lgu-67E0} for details.

\begin{itemize}
    \item To obtain the final answer used for accuracy evaluation, we employ a deterministic decoding strategy by disabling sampling (\texttt{do$\_$sample = False}) and setting the temperature to zero (temperature = 0.1). This ensures that the model always returns the highest-probability response, providing stable and reproducible outputs.
    \item For uncertainty estimation, we generate a set of candidate answers using stochastic sampling with (do\_sample = True) and a standard temperature setting (temperature = 1). This sampling-based approach produces diverse outputs whose variability reflects the model's predictive uncertainty. All answer generation prompt templates are referenced from Table~\ref{tab:prompt_templates}.
    \item Statistical Significance: We report 95\% confidence intervals for all evaluated metrics (AUROC, AUARC, ECE). These intervals are estimated using a bootstrapping procedure ($n_\text{bootstrap}=1000$, seed=42) over the test samples to ensure the robustness of our uncertainty estimations against data variance.
\end{itemize}


\paragraph{Baselines.}
We benchmark our method against a comprehensive suite of baselines, categorized by their underlying uncertainty mechanism.

\begin{itemize}
    \item \textbf{Semantic Uncertainty Methods:} These methods estimate uncertainty by considering the meaning of the generated text, rather than just its surface form.
    \begin{itemize}
        \item \textbf{Semantic Entropy}: Measures uncertainty over a distribution of semantically equivalent paraphrases for a given answer. It conceptually calculates the entropy of this distribution, where higher entropy implies greater semantic ambiguity.
        \item \textbf{Discrete Semantic Entropy}: A variant of semantic entropy that operates over a finite set of pre-defined semantic clusters instead of generating paraphrases, calculated without using the token-level log-likelihoods.
        \item \textbf{Kernel Language Entropy (KLE)} : A sophisticated method that estimates uncertainty using semantic similarity kernels and von Neumann entropy. It provides a more fine-grained semantic uncertainty quantification than standard semantic entropy. We adopt the default hyperparameters \((t = 0.3, \alpha = 0.5, \nu = 1, \kappa = 1)\) from the original work, which demonstrated robust performance.
    \end{itemize}

    \item \textbf{Token-level Uncertainty Methods:} These methods directly leverage the model's token-level output probabilities. Let \(y = (y_1, \dots, y_T)\) be a generated sequence of length \(T\) for an input \(x\). The probability of the sequence is given by the autoregressive factorization \(P(y|x) = \prod_{t=1}^{T} p(y_t | y_{<t}, x)\).
    \begin{itemize}
        \item \textbf{Maximum Sequence Probability (MSP)}: A direct measure of model confidence. The uncertainty is defined as one minus the sequence probability:
        \[
            \mathcal{U}_{\text{MSP}}(y|x) = 1 - P(y|x) = 1 - \prod_{t=1}^{T} p(y_t | y_{<t}, x)
        \]
        While simple, this metric can be insensitive as the product of many small probabilities quickly approaches zero.
        
        \item \textbf{Perplexity (PPL)}: Measures how well the probability distribution of a language model predicts a sample. It is the exponentiated average negative log-likelihood. A lower PPL indicates higher confidence.
        \[
            \text{PPL}(y|x) = \exp\left(-\frac{1}{T} \sum_{t=1}^{T} \log p(y_t | y_{<t}, x)\right)
        \]
        
        \item \textbf{Average and Max Negative Log-Probability}: These metrics, often denoted as \(\text{Avg}(-\log p)\) and \(\text{Max}(-\log p)\), isolate components of perplexity. \(\text{Avg}(-\log p)\) is the sequence-level cross-entropy loss, while \(\text{Max}(-\log p)\) identifies the point of maximum uncertainty within the sequence.
        \begin{align}
            \mathcal{U}_{\text{AvgNLL}}(y|x) &= -\frac{1}{T} \sum_{t=1}^{T} \log p(y_t | y_{<t}, x) \\
            \mathcal{U}_{\text{MaxNLL}}(y|x) &= \max_{1 \le t \le T} \left( -\log p(y_t | y_{<t}, x) \right)
        \end{align}
    \end{itemize}

    \item \textbf{Supervised \& Calibrated Methods:} These approaches require additional training or specialized prompting to estimate correctness.
    \begin{itemize}
        \item \textbf{Embedding Regression}: A supervised method that trains a simple classifier, such as a logistic regression model, on the final-layer hidden states (\(h_T\)) of the LLM to directly predict the probability of answer correctness.
    \end{itemize}
\end{itemize}


\paragraph{Evaluation Metrics.}
We assess all methods using three standard metrics, each evaluating a distinct aspect of uncertainty quality: discrimination, practical utility, and calibration.

\begin{itemize}
    \item \textbf{AUROC:} 
    Following prior work, we use AUROC to measure an uncertainty score's ability to rank incorrect predictions higher than correct ones. It plots the true positive rate against the false positive rate across all possible thresholds, with a score of 1.0 indicating perfect discrimination.

    \item \textbf{AUARC:}
    This metric evaluates the practical utility of uncertainty for selective prediction. It measures the area under the curve formed by plotting the accuracy on non-rejected samples versus the fraction of samples rejected, as the rejection threshold is varied. A higher AUARC signifies a better trade-off between accuracy and coverage.

    \item \textbf{ECE:}
    ECE measures the discrepancy between a model's confidence and its empirical accuracy, quantifying how well-calibrated the uncertainty scores are. Predictions are partitioned into \(M\) bins \(B_m\) based on their confidence scores. The ECE is the weighted average of the absolute difference between the average confidence and accuracy in each bin:
    \[
        \text{ECE} = \sum_{m=1}^{M} \frac{|B_m|}{N} \left| \text{acc}(B_m) - \text{conf}(B_m) \right|
    \]
    where \(N\) is the total number of samples, \(\text{acc}(B_m)\) is the fraction of correct predictions in bin \(B_m\), and \(\text{conf}(B_m)\) is the average confidence of predictions in that bin. A lower ECE indicates better calibration. 
    
    We adopt an empirical \textbf{min-max normalization protocol} based on the observed data range. While ECE was originally designed for probabilistic predictions, its computation only requires a confidence score in $[0,1]$ rather than strictly calibrated class probabilities.    
    In our setting, uncertainty estimates are mapped to confidence scores via the monotonic transformation $c = 1 - (u - u_{\min}) / (u_{\max} - u_{\min})$, which preserves the relative ordering of predictions while enabling calibration analysis. This evaluation protocol is standard in LLM uncertainty quantification; recent surveys explicitly identify ECE as a benchmark for assessing whether uncertainty estimates behave as reliable confidence signals, even when they are not calibrated probabilities 
    \cite{shorinwa2025survey}.
\end{itemize}

\subsection{Generalization Study}
\label{app:generate}

Table~\ref{tab:main_results}  presents the complete experimental results across all dataset-model pairs. The table provides a comprehensive comparison of each method under all evaluation metrics, further substantiating the robustness and generalizability of LGU. These results confirm that our approach consistently outperforms baselines in both semantic and token-level uncertainty estimation, demonstrating its effectiveness in diverse settings.

\paragraph{Average Rank Comparison across Models and Datasets.} While absolute metrics (AUROC, AUARC, and ECE) provide detailed performance indicators for individual model-dataset pairs, absolute scores can be heavily influenced by the inherent capabilities of different LLMs (e.g., modern models like LLaMA-3-8B naturally achieve a higher baseline AUROC than earlier models like Falcon-7B). To provide a rigorous, macro-level evaluation of robustness and to mitigate the variance introduced by disparate baseline performances, we conduct a comprehensive rank-based comparison.

For each specific evaluation setting (i.e., a combination of a model and a dataset), we rank the evaluated uncertainty estimation methods from best (rank 1) to worst. For discriminative metrics (AUROC and AUARC), higher numerical values receive better (lower) ranks; for calibration metrics (ECE), lower numerical values receive better ranks. We then average these ranks across all evaluated LLMs and benchmarks from our generalization study to compute the average rank per metric. The \textit{Overall Avg Rank} is the mean of the ranks across all three metrics. 

The average ranking results are summarized in Table \ref{tab:avg_rank}. We observe the following key findings:
\begin{itemize}
    \item \textbf{LGU Achieves State-of-the-Art Robustness:} Our proposed LGU achieves the best (lowest) average rank on every metric: AUROC (1.5417), AUARC (1.6250), and ECE (2.1250). With an Overall Avg Rank of 1.7639, LGU is the best-ranked method on average across the evaluated LLMs, ahead of the competitive KLE (2.0833).
    \item \textbf{Superiority over Strong Baselines:} Compared to the competitive KLE (Overall Rank: 2.0833), LGU demonstrates a clear advantage, particularly in calibration (ECE rank 2.1250 vs. 2.8750). This indicates that penalizing logical incompatibility not only improves discrimination but also prevents the over-confidence issues frequently observed in semantic-only methods.
    \item \textbf{Limitations of Traditional Methods:} The SE and NE lag significantly behind graph-based and kernel-based methods, confirming that merely measuring semantic equivalence or token-level uncertainty is insufficient. Furthermore, likelihood-based sequence methods such as PPL and representation-based methods such as Eigenvalue Ratio (ER) consistently rank at the bottom in this unsupervised setting, highlighting their lack of generalizability.
\end{itemize}

\begin{table*}[t!]
    \centering
    \renewcommand{\arraystretch}{1} 
    \caption{\textbf{Experimental Results for Falcon-7b, Falcon-7b-instruct, Llama-2-7b-chat, Mistral-7B-v0.1, Llama-2-13b-chat, Qwen3-32B, LLama-3-8B and Gemma-2-9B.} We report AUROC ($\uparrow$), AUARC ($\uparrow$), and ECE ($\downarrow$). Best results for each metric are highlighted in \textbf{bold}.}
    \label{tab:main_results}
    
    \resizebox{\textwidth}{!}{%
\begin{tabular}{c|l|ccc|ccc|ccc|ccc}
            \toprule
            
            & \multirow{2}{*}{\textbf{Method}} 
            & \multicolumn{3}{c|}{\textbf{Trivia QA}} 
            & \multicolumn{3}{c|}{\textbf{SQuAD}} 
            & \multicolumn{3}{c|}{\textbf{BioASQ}} 
            & \multicolumn{3}{c}{\textbf{NQ}} \\
            
            \cmidrule(lr){3-5} \cmidrule(lr){6-8} \cmidrule(lr){9-11} \cmidrule(lr){12-14}
            
            & 
            & \textbf{AUROC} & \textbf{AUARC} & \textbf{ECE} 
            & \textbf{AUROC} & \textbf{AUARC} & \textbf{ECE} 
            & \textbf{AUROC} & \textbf{AUARC} & \textbf{ECE} 
            & \textbf{AUROC} & \textbf{AUARC} & \textbf{ECE} \\
            \midrule
            
            \multirow{10}{*}{\rotatebox[origin=c]{90}{\textbf{Falcon-7b}}} 
            & SE & \eb{0.869}{0.037} & \eb{0.829}{0.031} & \textbf{\eb{0.046}{0.027}} & \eb{0.661}{0.050} & \eb{0.602}{0.053} & \eb{0.206}{0.045} & \eb{0.735}{0.047} & \eb{0.729}{0.045} & \eb{0.175}{0.041} & \eb{0.749}{0.048} & \eb{0.629}{0.054} & \eb{0.134}{0.041} \\
            & NE & \eb{0.809}{0.043} & \eb{0.803}{0.035} & \eb{0.069}{0.033} & \eb{0.601}{0.052} & \eb{0.565}{0.056} & \eb{0.087}{0.039} & \eb{0.748}{0.044} & \eb{0.734}{0.044} & \eb{0.093}{0.040} & \eb{0.684}{0.053} & \eb{0.590}{0.056} & \eb{0.122}{0.050} \\
            & MTE & \eb{0.750}{0.050} & \eb{0.777}{0.039} & \eb{0.269}{0.046} & \eb{0.613}{0.057} & \eb{0.565}{0.058} & \eb{0.365}{0.048} & \eb{0.697}{0.049} & \eb{0.712}{0.046} & \eb{0.306}{0.047} & \eb{0.628}{0.053} & \eb{0.554}{0.057} & \eb{0.411}{0.053} \\
            & MSP & \eb{0.753}{0.048} & \eb{0.781}{0.039} & \eb{0.253}{0.043} & \eb{0.625}{0.052} & \eb{0.566}{0.056} & \eb{0.289}{0.046} & \eb{0.705}{0.051} & \eb{0.715}{0.046} & \eb{0.249}{0.043} & \eb{0.624}{0.054} & \eb{0.554}{0.058} & \eb{0.372}{0.047} \\
\bottomrule
        \end{tabular}
    }
\end{table*}

\begin{table*}[t]
    \centering
    \caption{
    \textbf{Average ranking comparison of uncertainty estimation methods across models and datasets.}We report the average rank (lower is better) of each method across LLMs (e.g., LLaMA2, LLaMA3, Falcon) and benchmarks (e.g., SQuAD, TriviaQA), evaluated with common metrics. The overall average rank summarizes overall performance.\textbf{LGU overall achieves the best overall ranking}, demonstrating superior discrimination and calibration performance compared to both entropy-based and likelihood-based baselines.
    }
    \label{tab:avg_rank}
    \begin{tabular}{lcccc}
        \toprule
        \textbf{Method} & \textbf{AUROC} & \textbf{AUARC} & \textbf{ECE} & \textbf{Overall Avg Rank} \\
        \midrule
        \textbf{LGU}        & \textbf{1.5417} & \textbf{1.6250} & \textbf{2.1250} & \textbf{1.7639} \\
        KLE (full)          & 1.6250          & 1.7500          & 2.8750          & 2.0833          \\
        SE                  & 3.0833          & 2.9167          & 2.6667          & 2.8889          \\
        NE                  & 3.9583          & 3.8750          & 2.5833          & 3.4722          \\
        ER                  & 4.9167          & 5.0417          & 4.8750          & 4.9445          \\
        PPL                 & 5.8750          & 5.7917          & 5.8750          & 5.8472          \\
        \bottomrule
    \end{tabular}
\end{table*}

\subsection{Ablation Study on Incompatibility Measures}
\label{sec:ablation_study}

To scrutinize the contribution of each component in our framework and determine the optimal graph-density metric for the InS, we conduct a comprehensive ablation study. Table~\ref{tab:ablation_results} compares our proposed LGU against two reduced baselines and three alternative $\CS$ instantiations. The variants are defined as follows:

\begin{itemize}
    \item \textbf{Root Entropy (RE):} Computes the entropy over the probability distribution of the root hypotheses without considering implication chains or incompatibility penalties. This serves as a baseline for the logical graph structure.
    \item \textbf{IGE:} Aggregates probability mass along implication chains but calculates entropy over roots without penalizing mutual incompatibility (i.e., $\CS(x\mid L) = 0$).
    \item \textbf{LGU Variants (Different InS Metrics):} We test four implementations of the InS $\CS(x\mid L)$:
    \begin{itemize}
        \item \textbf{LGU (Edge Density):} Uses the directed edge density of the InG, capturing the normalized sparsity of conflicts.
        \item \textbf{LGUWD (Weighted Density):} Uses edge density weighted by the probability mass of conflicting nodes.
        \item \textbf{LGUAD (Average Degree):} Uses the average degree of nodes in the InG, reflecting the average number of conflicts per hypothesis.
        \item \textbf{LGUE (Estrada Index):} Uses a spectral graph measure based on the matrix exponential of the adjacency matrix, capturing global connectivity.
    \end{itemize}
\end{itemize}

\textbf{Analysis of Results.} 
First, we observe that explicitly modeling incompatibility is crucial. As shown in Table~\ref{tab:ablation_results}, LGU outperforms IGE across most datasets and models. For instance, on \textit{Falcon-7b-instruct} with the NQ dataset, LGU improves AUROC from 0.7349 (IGE) to 0.7517, confirming that penalizing mutually exclusive hypotheses provides a stronger signal for uncertainty than implication aggregation alone.

Second, regarding the choice of incompatibility metric, \textbf{LGU (Edge Density)} and \textbf{LGUWD (Weighted Density)} yield the most robust performance. While LGUWD achieves marginal gains in specific cases (e.g., TriviaQA on Falcon-7b), LGU (Edge Density) offers superior calibration (lower ECE) and stability across diverse tasks. 
In contrast, \textbf{LGUE (Estrada Index)} demonstrates significant instability. Although it captures global structure, its exponential nature makes it overly sensitive to graph density, leading to severe over-estimation of uncertainty and poor calibration (e.g., an ECE of 0.4646 on Falcon-7b/SQuAD compared to LGU's 0.0708). Similarly, \textbf{LGUAD} often results in higher ECE than density-based methods.

\textbf{Conclusion.} 
Based on these findings, we select \textbf{Directed Edge Density} as the default implementation for $\CS(x\mid L)$ in our final LGU framework. It provides the best trade-off between discriminative power (AUROC) and reliability (ECE), while maintaining mathematical simplicity and a bounded range $[0, 1]$.

\subsection{Additional Experiments and Analysis}

\subsubsection{Limitations of Evaluation Metrics and Bounds of ECE Normalization}
\label{app:metric}

\paragraph{Limitations of Using AUROC as the Sole Uncertainty Metric.}
Existing studies often rely on AUROC as the sole criterion for evaluating uncertainty estimation. However, AUROC only captures the ranking ability of uncertainty scores in distinguishing correct from incorrect predictions, and therefore fails to assess whether the numerical values of uncertainty are meaningfully calibrated. Since uncertainty fundamentally reflects a model’s confidence about its own predictions, proper calibration is essential. To address this limitation, we incorporate ECE as a complementary metric, providing a direct measure of the alignment between predicted uncertainty levels and actual error frequencies. The joint use of AUROC and ECE yields a more complete and reliable assessment of uncertainty quality.

\paragraph{Robustness to ECE Normalization Protocols: Theoretical vs. Empirical Bounds.} Evaluating Expected Calibration Error (ECE) requires transforming unbounded uncertainty measures (e.g., semantic entropy or graph-based uncertainty) into valid confidence scores within $[0, 1]$. This is typically achieved via min-max normalization: $\text{Confidence} = 1 - \frac{U - U_{\min}}{U_{\max} - U_{\min}}$. However, the choice of the upper bound ($U_{\max}$) and lower bound ($U_{\min}$) can be determined in two distinct ways:

\begin{enumerate}
    \item \textbf{Theoretical Bounds:} Derived from the mathematical definitions of the metrics. For example, Shannon entropy over $N$ candidate answers is bounded by $\log(N)$, and the theoretical maximum of our LGU is strictly delineated by Theorem~\ref{thm:EEextreme}.
    \item \textbf{Empirical (Actual) Bounds:} Derived directly from the observed dynamic range (i.e., the actual minimum and maximum values generated by a specific model on a specific dataset).
\end{enumerate}

To ensure our calibration evaluation is not an artifact of the normalization strategy, we conduct a comprehensive comparison of the best-performing methods (i.e., the method achieving the lowest ECE) under both theoretical and empirical bounding protocols.

\begin{table*}[ht]
    \centering
    \renewcommand{\arraystretch}{1.1}
    \caption{\textbf{Best uncertainty estimation methods under theoretical and empirical ECE evaluation across models and datasets.} For each model--dataset pair, we report the method achieving the lowest Expected Calibration Error (ECE) under the theoretical upper bound and the empirical estimation. \textbf{We observe a substantial agreement between the two evaluation protocols, while LGU is selected most frequently under both settings.} Specifically, under the theoretical upper bound, LGU is selected 21 times, followed by NE (10) and SE (9). Under empirical estimation, LGU remains dominant with 19 selections, followed by NE (10), SE (6), and KLE (5), demonstrating the robustness and consistency of LGU across diverse models and datasets.}
    \label{tab:ece_bounds}
    \begin{tabular}{llcc}
        \toprule
        \textbf{Model} & \textbf{Dataset} & \textbf{Best (Theory)} & \textbf{Best (Actual)} \\
        \midrule
        \multirow{4}{*}{LLaMA2-13B-chat} 
        & BioASQ   & LGU & LGU \\
        & NQ       & SE  & SE  \\
        & SQuAD    & NE  & NE  \\
        & TriviaQA & SE  & KLE \\
        \midrule
        \multirow{4}{*}{LLaMA2-7B-chat} 
        & BioASQ   & SE  & NE  \\
        & NQ       & SE  & SE  \\
        & SQuAD    & LGU & LGU \\
        & TriviaQA & NE  & KLE \\
        \midrule
        \multirow{4}{*}{LLaMA3-8B} 
        & BioASQ   & NE  & LGU \\
        & NQ       & NE  & NE  \\
        & SQuAD    & LGU & LGU \\
        & TriviaQA & LGU & NE  \\
        \midrule
        \multirow{4}{*}{LLaMA3-8B-4bit} 
        & BioASQ   & LGU & NE  \\
        & NQ       & NE  & NE  \\
        & SQuAD    & LGU & LGU \\
        & TriviaQA & LGU & LGU \\
        \midrule
        \multirow{4}{*}{LLaMA3-8B-8bit} 
        & BioASQ   & LGU & LGU \\
        & NQ       & NE  & NE  \\
        & SQuAD    & LGU & LGU \\
        & TriviaQA & SE  & SE  \\
        \midrule
        \multirow{4}{*}{Mistral-7B-v0.1} 
        & BioASQ   & LGU & LGU \\
        & NQ       & LGU & NE  \\
        & SQuAD    & LGU & LGU \\
        & TriviaQA & SE  & SE  \\
        \midrule
        \multirow{4}{*}{Qwen3-32B} 
        & BioASQ   & LGU & LGU \\
        & NQ       & LGU & LGU \\
        & SQuAD    & LGU & LGU \\
        & TriviaQA & NE  & KLE \\
        \midrule
        \multirow{4}{*}{Falcon-7B} 
        & BioASQ   & LGU & LGU \\
        & NQ       & LGU & LGU \\
        & SQuAD    & LGU & LGU \\
        & TriviaQA & SE  & SE  \\
        \midrule
        \multirow{4}{*}{Falcon-7B-instruct} 
        & BioASQ   & NE  & NE  \\
        & NQ       & NE  & SE  \\
        & SQuAD    & LGU & LGU \\
        & TriviaQA & NE  & KLE \\
        \midrule
        \multirow{4}{*}{Gemma-2-9B-it} 
        & BioASQ   & LGU & LGU \\
        & NQ       & SE  & NE  \\
        & SQuAD    & LGU & LGU \\
        & TriviaQA & SE  & KLE \\
        \bottomrule
    \end{tabular}
\end{table*}

Table \ref{tab:ece_bounds} summarizes the winning methods across 40 model-dataset pairs. We observe a substantial agreement between the two evaluation protocols, confirming the reliability of our ECE findings. Most importantly, \textbf{LGU is selected most frequently under both settings}. Specifically, under the theoretical upper bound, LGU achieves the lowest ECE in 21 out of 40 cases, followed by NE (10) and SE (9). Under the empirical estimation, LGU remains the dominant method with 19 selections, followed by NE (10), SE (6), and KLE (5). 

Given that empirical bounds adapt naturally to the actual dynamic range of model outputs without requiring explicit mathematical derivation for every arbitrary score formulation, and since the relative superiority of the methods remains consistent, we adopt the empirical bounds for the main experiments in this paper. This ablation highlights the robust calibration capabilities of LGU across diverse generation scenarios.

\subsubsection{Empirical Verification of NLI-based Logical Properties.}
\label{app:nli_verification}

A critical premise of our proposed framework is that the off-the-shelf NLI model (DeBERTa-Large-MNLI) can reliably capture formal logical relations to construct valid graph structures. To validate the structural integrity of our logical graphs and support our theoretical claims, we conduct large-scale empirical property testing, detailed as follows.

\begin{enumerate}
    \item \textbf{Implication Graph and Transitivity.} According to Proposition~\ref{pro:entailment}, the logical implication relation utilized to construct the ImG must be reflexive and transitive. In our implementation, reflexivity holds trivially and does not affect the computation of IGE, as self-loops are omitted by definition. However, transitivity is paramount for treating the ImG as a Directed Acyclic Graph (DAG) (Theorem~\ref{thm:dag}). To empirically verify this, we randomly sampled 4,863 implication chains from our generated data (i.e., triplets where the NLI model confidently predicts $A \Rightarrow B$ and $B \Rightarrow C$). Among these chains, the NLI model correctly predicted $A \Rightarrow C$ in \textbf{97.0\%} of the cases. This remarkably high transitivity rate empirically justifies our implication-based probability aggregation and confirms that the NLI model provides a highly reliable directed topology for the ImG.

    \item \textbf{Incompatibility Graph and Symmetry.} Proposition~\ref{pro:incompatibility} defines incompatibility as a symmetric, non-reflexive, and non-transitive relation. Non-reflexivity and non-transitivity hold trivially for mutually exclusive statements (e.g., a hypothesis cannot contradict itself). To verify symmetry, we randomly sampled 5,339 incompatible pairs predicted by the NLI model (i.e., hypothesis $A$ contradicts hypothesis $B$). We observed that \textbf{92.8\%} of these pairs were strictly symmetric (i.e., the model also predicts $B$ contradicts $A$). This robust symmetry ensures that our incompatibility penalty (InS), which relies on an undirected InG graph-density metric (e.g., Edge Density), operates on a mathematically sound foundation.
\end{enumerate}

\subsubsection{Impact of Model Quantization on Uncertainty Estimation}

In real-world deployments, large language models are frequently subjected to weight quantization (e.g., 8-bit or 4-bit precision) to reduce memory footprints and accelerate inference. However, quantization introduces noise into the model's intermediate representations and output logits, which can potentially distort the token-level probability distributions that many uncertainty estimation methods rely on. To evaluate the robustness of our proposed framework against precision degradation, we conduct an additional experiment comparing the performance of uncertainty metrics across three precision levels of the LLaMA-3-8B model: Full Precision (BF16), 8-bit, and 4-bit. The evaluation results are presented in Table \ref{tab:quantization}. This experiment confirms that LGU is not only highly effective in standard settings but is also exceptionally practical for resource-constrained, low-precision deployment scenarios.

\begin{table*}[t!]
    \centering
    \renewcommand{\arraystretch}{1.1}
    \setlength{\tabcolsep}{3pt} 
    \tiny 
    \caption{\textbf{Performance comparison of uncertainty estimation methods across different model precisions.} We report AUROC ($\uparrow$), AUARC ($\uparrow$), and ECE ($\downarrow$) using the LLaMA-3-8B model in Full Precision, 4-bit, and 8-bit configurations. }
    \label{tab:quantization}
    
    \resizebox{\textwidth}{!}{%
        \begin{tabular}{c|l|ccc|ccc|ccc|ccc}
            \toprule
            
            & \multirow{2}{*}{\textbf{Method}} 
            & \multicolumn{3}{c|}{\textbf{Trivia QA}} 
            & \multicolumn{3}{c|}{\textbf{SQuAD}} 
            & \multicolumn{3}{c|}{\textbf{BioASQ}} 
            & \multicolumn{3}{c}{\textbf{NQ}} \\
            
            \cmidrule(lr){3-5} \cmidrule(lr){6-8} \cmidrule(lr){9-11} \cmidrule(lr){12-14}
            
            & 
            & \textbf{AUROC} & \textbf{AUARC} & \textbf{ECE} 
            & \textbf{AUROC} & \textbf{AUARC} & \textbf{ECE} 
            & \textbf{AUROC} & \textbf{AUARC} & \textbf{ECE} 
            & \textbf{AUROC} & \textbf{AUARC} & \textbf{ECE} \\
            \midrule
            
            \multirow{10}{*}{\rotatebox[origin=c]{90}{\textbf{LLaMA3-8B}}} 
            & SE & \eb{0.8739}{0.0456} & \eb{0.8860}{0.0233} & \eb{0.0661}{0.0265} & \eb{0.6551}{0.0507} & \eb{0.7075}{0.0463} & \eb{0.2868}{0.0465} & \eb{0.7776}{0.0462} & \eb{0.8195}{0.0311} & \eb{0.1929}{0.0388} & \eb{0.7165}{0.0520} & \eb{0.7247}{0.0461} & \eb{0.2278}{0.0428} \\
            & NE & \eb{0.8586}{0.0467} & \eb{0.8848}{0.0218} & \textbf{\eb{0.0586}{0.0301}} & \eb{0.6190}{0.0524} & \eb{0.6819}{0.0500} & \eb{0.1744}{0.0454} & \eb{0.7758}{0.0490} & \eb{0.8159}{0.0328} & \eb{0.0479}{0.0437} & \eb{0.7227}{0.0550} & \eb{0.7272}{0.0478} & \textbf{\eb{0.0379}{0.0335}} \\
            & MTE & \eb{0.5568}{0.0458} & \eb{0.7758}{0.0376} & \eb{0.1879}{0.0384} & \eb{0.5900}{0.0542} & \eb{0.6449}{0.0506} & \eb{0.2731}{0.0479} & \eb{0.6558}{0.0546} & \eb{0.7453}{0.0414} & \eb{0.2177}{0.0416} & \eb{0.6063}{0.0485} & \eb{0.6547}{0.0475} & \eb{0.2750}{0.0437} \\
            & MSP & \eb{0.5567}{0.0463} & \eb{0.7756}{0.0375} & \eb{0.1980}{0.0385} & \eb{0.5867}{0.0531} & \eb{0.6435}{0.0507} & \eb{0.2624}{0.0439} & \eb{0.6603}{0.0543} & \eb{0.7470}{0.0412} & \eb{0.2075}{0.0416} & \eb{0.6058}{0.0483} & \eb{0.6544}{0.0473} & \eb{0.2788}{0.0426} \\
            & Avg($-\log p$) & \eb{0.5567}{0.0464} & \eb{0.7756}{0.0376} & \eb{0.2027}{0.0387} & \eb{0.5903}{0.0536} & \eb{0.6446}{0.0508} & \eb{0.2940}{0.0468} & \eb{0.6561}{0.0547} & \eb{0.7453}{0.0416} & \eb{0.2348}{0.0426} & \eb{0.6102}{0.0490} & \eb{0.6555}{0.0471} & \eb{0.3199}{0.0457} \\
            & Max($-\log p$) & \eb{0.5567}{0.0463} & \eb{0.7756}{0.0375} & \eb{0.1938}{0.0388} & \eb{0.5836}{0.0533} & \eb{0.6420}{0.0502} & \eb{0.2847}{0.0471} & \eb{0.6602}{0.0550} & \eb{0.7470}{0.0415} & \eb{0.2285}{0.0411} & \eb{0.6058}{0.0478} & \eb{0.6545}{0.0476} & \eb{0.2851}{0.0435} \\
            & Embedding Regression & \eb{0.7909}{0.0581} & \eb{0.8559}{0.0317} & \eb{0.1276}{0.0312} & \eb{0.5676}{0.0576} & \eb{0.6409}{0.0534} & \eb{0.3498}{0.0459} & \eb{0.6805}{0.0594} & \eb{0.7570}{0.0450} & \eb{0.2089}{0.0372} & \eb{0.6097}{0.0544} & \eb{0.6688}{0.0508} & \eb{0.3062}{0.0440} \\
            & PPL & \eb{0.5567}{0.0464} & \eb{0.7756}{0.0376} & \eb{0.1994}{0.0393} & \eb{0.5903}{0.0536} & \eb{0.6446}{0.0508} & \eb{0.3133}{0.0474} & \eb{0.6561}{0.0547} & \eb{0.7453}{0.0416} & \eb{0.2469}{0.0424} & \eb{0.6102}{0.0490} & \eb{0.6555}{0.0471} & \eb{0.3313}{0.0454} \\
            & KLE(Full) & \eb{0.8784}{0.0474} & \eb{0.8848}{0.0247} & \eb{0.1121}{0.0313} & \textbf{\eb{0.7148}{0.0494}} & \eb{0.7255}{0.0449} & \eb{0.3266}{0.0448} & \eb{0.7898}{0.0438} & \eb{0.8237}{0.0295} & \eb{0.2162}{0.0411} & \textbf{\eb{0.7351}{0.0495}} & \textbf{\eb{0.7343}{0.0451}} & \eb{0.2806}{0.0419} \\
            \cmidrule{2-14}
            & \textbf{LGU (Ours)} & \textbf{\eb{0.8845}{0.0454}} & \textbf{\eb{0.8873}{0.0230}} & \eb{0.0599}{0.0264} & \eb{0.7141}{0.0486} & \textbf{\eb{0.7285}{0.0458}} & \textbf{\eb{0.0895}{0.0379}} & \textbf{\eb{0.8143}{0.0416}} & \textbf{\eb{0.8317}{0.0278}} & \textbf{\eb{0.0376}{0.0280}} & \eb{0.7259}{0.0522} & \eb{0.7316}{0.0451} & \eb{0.1028}{0.0364} \\
            \midrule
            \multirow{10}{*}{\rotatebox[origin=c]{90}{\textbf{LLaMA3-8B-8bit}}} 
            & SE & \eb{0.8636}{0.0411} & \eb{0.8785}{0.0232} & \textbf{\eb{0.0348}{0.0243}} & \eb{0.6995}{0.0517} & \eb{0.6886}{0.0507} & \eb{0.2460}{0.0482} & \eb{0.7727}{0.0492} & \eb{0.8186}{0.0341} & \eb{0.2069}{0.0375} & \eb{0.7449}{0.0511} & \eb{0.7460}{0.0457} & \eb{0.2438}{0.0410} \\
            & NE & \eb{0.8359}{0.0419} & \eb{0.8761}{0.0216} & \eb{0.0456}{0.0285} & \eb{0.6700}{0.0509} & \eb{0.6702}{0.0525} & \eb{0.1094}{0.0447} & \eb{0.7706}{0.0490} & \eb{0.8190}{0.0346} & \eb{0.0738}{0.0390} & \eb{0.7307}{0.0528} & \eb{0.7428}{0.0432} & \textbf{\eb{0.0536}{0.0380}} \\
            & MTE & \eb{0.5917}{0.0479} & \eb{0.7699}{0.0382} & \eb{0.1866}{0.0389} & \eb{0.6197}{0.0503} & \eb{0.6164}{0.0532} & \eb{0.3068}{0.0504} & \eb{0.6549}{0.0550} & \eb{0.7576}{0.0400} & \eb{0.2055}{0.0382} & \eb{0.6444}{0.0519} & \eb{0.6822}{0.0462} & \eb{0.2494}{0.0429} \\
            & MSP & \eb{0.5938}{0.0478} & \eb{0.7701}{0.0382} & \eb{0.1841}{0.0399} & \eb{0.6241}{0.0502} & \eb{0.6175}{0.0535} & \eb{0.2632}{0.0460} & \eb{0.6563}{0.0551} & \eb{0.7574}{0.0401} & \eb{0.2078}{0.0391} & \eb{0.6316}{0.0502} & \eb{0.6780}{0.0468} & \eb{0.2558}{0.0414} \\
            & Avg($-\log p$) & \eb{0.5930}{0.0479} & \eb{0.7699}{0.0382} & \eb{0.2005}{0.0392} & \eb{0.6243}{0.0506} & \eb{0.6174}{0.0535} & \eb{0.3434}{0.0493} & \eb{0.6563}{0.0550} & \eb{0.7585}{0.0396} & \eb{0.2085}{0.0374} & \eb{0.6404}{0.0506} & \eb{0.6811}{0.0460} & \eb{0.3192}{0.0468} \\
            & Max($-\log p$) & \eb{0.5938}{0.0476} & \eb{0.7702}{0.0382} & \eb{0.1908}{0.0400} & \eb{0.6237}{0.0504} & \eb{0.6172}{0.0535} & \eb{0.3161}{0.0511} & \eb{0.6563}{0.0553} & \eb{0.7575}{0.0398} & \eb{0.2122}{0.0394} & \eb{0.6311}{0.0497} & \eb{0.6778}{0.0466} & \eb{0.2844}{0.0416} \\
            & Embedding Regression & \eb{0.8059}{0.0532} & \eb{0.8525}{0.0327} & \eb{0.1426}{0.0333} & \eb{0.6109}{0.0562} & \eb{0.6197}{0.0577} & \eb{0.3518}{0.0459} & \eb{0.6767}{0.0597} & \eb{0.7765}{0.0430} & \eb{0.2336}{0.0431} & \eb{0.6230}{0.0524} & \eb{0.6792}{0.0493} & \eb{0.3097}{0.0447} \\
            & PPL & \eb{0.5930}{0.0479} & \eb{0.7699}{0.0382} & \eb{0.2058}{0.0395} & \eb{0.6243}{0.0506} & \eb{0.6174}{0.0535} & \eb{0.3658}{0.0494} & \eb{0.6563}{0.0550} & \eb{0.7585}{0.0396} & \eb{0.2151}{0.0391} & \eb{0.6404}{0.0506} & \eb{0.6811}{0.0460} & \eb{0.3365}{0.0483} \\
            & KLE(Full) & \eb{0.8788}{0.0386} & \eb{0.8819}{0.0215} & \eb{0.0930}{0.0302} & \eb{0.7370}{0.0496} & \eb{0.7044}{0.0497} & \eb{0.2817}{0.0474} & \eb{0.7923}{0.0468} & \eb{0.8265}{0.0327} & \eb{0.2510}{0.0392} & \textbf{\eb{0.7637}{0.0511}} & \textbf{\eb{0.7546}{0.0447}} & \eb{0.2893}{0.0421} \\
            \cmidrule{2-14}
            & \textbf{LGU (Ours)} & \textbf{\eb{0.8879}{0.0379}} & \textbf{\eb{0.8843}{0.0224}} & \eb{0.0411}{0.0249} & \textbf{\eb{0.7423}{0.0482}} & \textbf{\eb{0.7100}{0.0488}} & \textbf{\eb{0.0426}{0.0323}} & \textbf{\eb{0.7987}{0.0444}} & \textbf{\eb{0.8302}{0.0307}} & \textbf{\eb{0.0540}{0.0309}} & \eb{0.7385}{0.0522} & \eb{0.7428}{0.0441} & \eb{0.1071}{0.0365} \\
            \midrule
            \multirow{10}{*}{\rotatebox[origin=c]{90}{\textbf{LLaMA3-8B-4bit}}} 
            & SE & \eb{0.8819}{0.0421} & \eb{0.8765}{0.0233} & \eb{0.0626}{0.0270} & \eb{0.7246}{0.0487} & \eb{0.6942}{0.0483} & \eb{0.2584}{0.0478} & \eb{0.7623}{0.0489} & \eb{0.8135}{0.0349} & \eb{0.2121}{0.0400} & \eb{0.7485}{0.0480} & \eb{0.7198}{0.0462} & \eb{0.1988}{0.0422} \\
            & NE & \eb{0.8511}{0.0436} & \eb{0.8713}{0.0234} & \eb{0.0609}{0.0293} & \eb{0.6453}{0.0525} & \eb{0.6619}{0.0515} & \eb{0.1275}{0.0438} & \eb{0.7459}{0.0519} & \eb{0.8093}{0.0355} & \textbf{\eb{0.0491}{0.0330}} & \eb{0.7259}{0.0493} & \eb{0.7094}{0.0480} & \textbf{\eb{0.0176}{0.0423}} \\
            & MTE & \eb{0.6149}{0.0504} & \eb{0.7698}{0.0389} & \eb{0.2008}{0.0392} & \eb{0.6203}{0.0493} & \eb{0.6162}{0.0534} & \eb{0.3112}{0.0493} & \eb{0.6411}{0.0542} & \eb{0.7498}{0.0419} & \eb{0.2206}{0.0417} & \eb{0.6103}{0.0498} & \eb{0.6308}{0.0487} & \eb{0.3143}{0.0532} \\
            & MSP & \eb{0.6142}{0.0505} & \eb{0.7695}{0.0388} & \eb{0.2002}{0.0393} & \eb{0.6265}{0.0507} & \eb{0.6177}{0.0537} & \eb{0.2780}{0.0453} & \eb{0.6550}{0.0568} & \eb{0.7529}{0.0419} & \eb{0.2041}{0.0408} & \eb{0.6129}{0.0499} & \eb{0.6322}{0.0485} & \eb{0.2891}{0.0451} \\
            & Avg($-\log p$) & \eb{0.6144}{0.0508} & \eb{0.7697}{0.0386} & \eb{0.2169}{0.0414} & \eb{0.6237}{0.0495} & \eb{0.6171}{0.0546} & \eb{0.3595}{0.0504} & \eb{0.6457}{0.0550} & \eb{0.7509}{0.0421} & \eb{0.2318}{0.0409} & \eb{0.6069}{0.0498} & \eb{0.6302}{0.0485} & \eb{0.3347}{0.0458} \\
            & Max($-\log p$) & \eb{0.6139}{0.0506} & \eb{0.7694}{0.0389} & \eb{0.2176}{0.0403} & \eb{0.6242}{0.0503} & \eb{0.6175}{0.0535} & \eb{0.3336}{0.0484} & \eb{0.6545}{0.0565} & \eb{0.7525}{0.0419} & \eb{0.1964}{0.0394} & \eb{0.6105}{0.0498} & \eb{0.6312}{0.0487} & \eb{0.3217}{0.0467} \\
            & Embedding Regression & \eb{0.8028}{0.0508} & \eb{0.8488}{0.0292} & \eb{0.1547}{0.0347} & \eb{0.6454}{0.0554} & \eb{0.6385}{0.0581} & \eb{0.3042}{0.0448} & \eb{0.6768}{0.0586} & \eb{0.7686}{0.0421} & \eb{0.2311}{0.0414} & \eb{0.6404}{0.0563} & \eb{0.6578}{0.0537} & \eb{0.3086}{0.0483} \\
            & PPL & \eb{0.6144}{0.0508} & \eb{0.7697}{0.0386} & \eb{0.2214}{0.0419} & \eb{0.6237}{0.0495} & \eb{0.6171}{0.0546} & \eb{0.3785}{0.0521} & \eb{0.6457}{0.0550} & \eb{0.7509}{0.0421} & \eb{0.2426}{0.0420} & \eb{0.6069}{0.0498} & \eb{0.6302}{0.0485} & \eb{0.3504}{0.0475} \\
            & KLE(Full) & \textbf{\eb{0.8888}{0.0416}} & \textbf{\eb{0.8771}{0.0239}} & \eb{0.1137}{0.0290} & \textbf{\eb{0.7587}{0.0438}} & \textbf{\eb{0.7123}{0.0473}} & \eb{0.2924}{0.0469} & \eb{0.7782}{0.0495} & \eb{0.8184}{0.0337} & \eb{0.2605}{0.0403} & \textbf{\eb{0.7657}{0.0470}} & \textbf{\eb{0.7311}{0.0450}} & \eb{0.2669}{0.0423} \\
            \cmidrule{2-14}
            & \textbf{LGU (Ours)} & \eb{0.8828}{0.0420} & \eb{0.8739}{0.0246} & \textbf{\eb{0.0445}{0.0240}} & \eb{0.7518}{0.0470} & \eb{0.7066}{0.0475} & \textbf{\eb{0.0380}{0.0406}} & \textbf{\eb{0.8064}{0.0464}} & \textbf{\eb{0.8303}{0.0319}} & \eb{0.0515}{0.0320} & \eb{0.7506}{0.0475} & \eb{0.7176}{0.0491} & \eb{0.0864}{0.0353} \\
            \midrule
            
            \bottomrule
        \end{tabular}
    }
\end{table*}

\subsubsection{Black-box Uncertainty Estimation}

In addition to white-box methods, we propose a quantitative uncertainty estimation approach tailored for black-box large language models. Table~\ref{tab:black_box_full_results} compares our Discrete Logical Graph Uncertainty (DLGU) method with two competitive baselines: Discrete Semantic Entropy (DSE) and Kernel Language Entropy using a heat kernel (KLE(heat)).

DLGU consistently outperforms DSE across nearly all models and datasets, demonstrating stronger discriminative power. Compared to the state-of-the-art KLE(heat), DLGU achieves comparable or better results, especially in calibration. DLGU attains the lowest ECE in most settings, indicating more reliable uncertainty estimates. For example, on the Falcon-7b model with Trivia QA, DLGU reduces ECE to 0.070, substantially lower than KLE(heat)'s 0.183. On discriminative metrics such as AUROC and AUARC, DLGU remains highly competitive, frequently matching or surpassing KLE(heat). These results validate the effectiveness and robustness of DLGU for black-box uncertainty quantification.

\subsubsection{Case Study: Logical Entailment and Contradiction Detection}
\label{app:case_study}

Table \ref{tab:case_study_portrait} presents a qualitative comparison between our proposed LGU method and baseline metrics (SE and KLE) across three distinct scenarios: consistent correctness, incorrect prediction, and severe hallucination.

\textbf{Identifying Logical Entailment.} The first row demonstrates a case where the model correctly answers "Apollo". Although the sampled responses vary lexically (e.g., "the god of prophecy, apollo" vs. "Apollo"), they are semantically equivalent. While KLE assigns a relatively high uncertainty score (0.89) due to token-level discrepancies, LGU successfully recognizes the logical entailment among these responses, yielding a perfect uncertainty score of 0. This indicates that LGU is robust to surface-level variations when the underlying logic remains consistent.

\textbf{Penalizing Severe Hallucinations.} The third row illustrates a scenario of severe hallucination, where the model generates a "confabulation" regarding the number of times Lake Placid has hosted the Olympics. The sampled responses are riddled with logical contradictions, ranging from "twice" to "Ten times," with various invented years. In this case, LGU calculates a significantly higher uncertainty score (3.42) compared to SE (2.26) and KLE (2.27). This sharp increase demonstrates LGU's capability to detect and penalize high-entropy logical conflicts, effectively flagging instances where the large language model is "talking nonsense" despite potential confidence in its tone.

\begin{table*}[t!]
    \centering
    \renewcommand{\arraystretch}{1.1}
    \setlength{\tabcolsep}{3pt} 
    \tiny 
    \caption{\textbf{Ablation Study of Logical Graph Uncertainty Components.} We compare RE, IGE, and LGU variants (LGU, LGUWD, LGUAD, LGUE) across different models and datasets. We report AUROC ($\uparrow$), AUARC ($\uparrow$), and ECE ($\downarrow$). Values are rounded to 4 decimal places. Best results for each model are highlighted in \textbf{bold}.}
    \label{tab:ablation_results}
    
    \resizebox{\textwidth}{!}{%
\begin{tabular}{c|l|ccc|ccc|ccc|ccc}
            \toprule
            
            & \multirow{2}{*}{\textbf{Method}} 
            & \multicolumn{3}{c|}{\textbf{Trivia QA}} 
            & \multicolumn{3}{c|}{\textbf{SQuAD}} 
            & \multicolumn{3}{c|}{\textbf{BioASQ}} 
            & \multicolumn{3}{c}{\textbf{NQ}} \\
            
            \cmidrule(lr){3-5} \cmidrule(lr){6-8} \cmidrule(lr){9-11} \cmidrule(lr){12-14}
            
            & 
            & \textbf{AUROC} & \textbf{AUARC} & \textbf{ECE} 
            & \textbf{AUROC} & \textbf{AUARC} & \textbf{ECE} 
            & \textbf{AUROC} & \textbf{AUARC} & \textbf{ECE} 
            & \textbf{AUROC} & \textbf{AUARC} & \textbf{ECE} \\
            \midrule
            
            \multirow{7}{*}{\rotatebox[origin=c]{90}{\textbf{Falcon-7b}}} 
            & RE & \eb{0.8797}{0.0333} & \eb{0.8359}{0.0284} & \textbf{\eb{0.0484}{0.0256}} & \eb{0.6622}{0.0511} & \eb{0.5989}{0.0544} & \eb{0.1545}{0.0425} & \eb{0.7540}{0.0466} & \eb{0.7295}{0.0463} & \eb{0.1473}{0.0398} & \eb{0.7559}{0.0476} & \eb{0.6347}{0.0528} & \eb{0.0764}{0.0370} \\
            & IGE & \eb{0.8793}{0.0349} & \eb{0.8348}{0.0288} & \eb{0.0589}{0.0279} & \eb{0.6626}{0.0504} & \eb{0.6027}{0.0542} & \eb{0.1302}{0.0399} & \eb{0.7688}{0.0441} & \eb{0.7409}{0.0450} & \eb{0.1023}{0.0359} & \eb{0.7592}{0.0458} & \eb{0.6349}{0.0522} & \eb{0.0609}{0.0340} \\
            & \textbf{LGU (Ours)} & \eb{0.8830}{0.0344} & \eb{0.8367}{0.0292} & \eb{0.0500}{0.0289} & \eb{0.6968}{0.0484} & \eb{0.6197}{0.0517} & \textbf{\eb{0.0708}{0.0346}} & \textbf{\eb{0.7703}{0.0433}} & \textbf{\eb{0.7429}{0.0442}} & \textbf{\eb{0.0569}{0.0328}} & \eb{0.7619}{0.0468} & \eb{0.6389}{0.0542} & \textbf{\eb{0.0572}{0.0342}} \\
\bottomrule
        \end{tabular}
    }
\end{table*}

\begin{table*}[t!]
\centering
\renewcommand{\arraystretch}{1.05} 
\caption{\textbf{Comparison of Black-box Uncertainty Estimation Methods.} The table reports AUROC ($\uparrow$), AUARC ($\uparrow$), and ECE ($\downarrow$). We compare Discrete Semantic Entropy (DSE), KLE using heat kernel (KLE(heat)), and variants of our graph-based measures: DRE, DIGE, DLGU, DLGUWD, DLGUAD, and DLGUE. Best results for each model are highlighted in \textbf{bold}.}
\label{tab:black_box_full_results}

\resizebox{\textwidth}{!}{%
    \begin{tabular}{c|l|ccc|ccc|ccc|ccc}
        \toprule
        
        & \multirow{2}{*}{\textbf{Method}} 
        & \multicolumn{3}{c|}{\textbf{Trivia QA}} 
        & \multicolumn{3}{c|}{\textbf{SQuAD}} 
        & \multicolumn{3}{c|}{\textbf{BioASQ}} 
        & \multicolumn{3}{c}{\textbf{NQ}} \\
        
        \cmidrule(lr){3-5} \cmidrule(lr){6-8} \cmidrule(lr){9-11} \cmidrule(lr){12-14}
        
        & 
        & \textbf{AUROC} & \textbf{AUARC} & \textbf{ECE} 
        & \textbf{AUROC} & \textbf{AUARC} & \textbf{ECE} 
        & \textbf{AUROC} & \textbf{AUARC} & \textbf{ECE} 
        & \textbf{AUROC} & \textbf{AUARC} & \textbf{ECE} \\
        \midrule
        
            \multirow{10}{*}{\rotatebox[origin=c]{90}{\textbf{Falcon-7b}}} 
            & DSE & \eb{0.8610}{0.0378} & \eb{0.8197}{0.0309} & \eb{0.1600}{0.0371} & \eb{0.6579}{0.0501} & \eb{0.5908}{0.0526} & \eb{0.3322}{0.0483} & \eb{0.7296}{0.0479} & \eb{0.7089}{0.0465} & \eb{0.2484}{0.0429} & \eb{0.7596}{0.0457} & \eb{0.6160}{0.0517} & \eb{0.2533}{0.0450} \\
            & KLE (heat) & \textbf{\eb{0.8811}{0.0360}} & \textbf{\eb{0.8329}{0.0301}} & \eb{0.1830}{0.0380} & \textbf{\eb{0.7181}{0.0501}} & \textbf{\eb{0.6297}{0.0531}} & \eb{0.2669}{0.0464} & \textbf{\eb{0.7793}{0.0428}} & \textbf{\eb{0.7525}{0.0424}} & \eb{0.2131}{0.0409} & \textbf{\eb{0.7776}{0.0448}} & \textbf{\eb{0.6419}{0.0522}} & \eb{0.2390}{0.0430} \\
            & DRE & \eb{0.8662}{0.0361} & \eb{0.8279}{0.0295} & \eb{0.1103}{0.0351} & \eb{0.6647}{0.0510} & \eb{0.5929}{0.0536} & \eb{0.2326}{0.0467} & \eb{0.7535}{0.0457} & \eb{0.7252}{0.0463} & \eb{0.1806}{0.0402} & \eb{0.7581}{0.0454} & \eb{0.6238}{0.0533} & \eb{0.1822}{0.0418} \\
\bottomrule
    \end{tabular}
}
\end{table*}

\begin{table*}[t] 
    \centering
    \scriptsize 
    \renewcommand{\arraystretch}{1.3}
    \setlength{\tabcolsep}{4pt}
    
    \caption{\textbf{Case Study on Uncertainty Estimation.} We compare Semantic Entropy (SE), Kernel Language Entropy (KLE), and our LGU. The ``Responses Cluster'' column shows the semantic clusters formed from multiple sampled responses. LGU successfully identifies logical entailment in the first case (uncertainty $\approx$ 0) and aggressively penalizes contradictions in the third case.}
    \label{tab:case_study_portrait}

    \begin{tabularx}{\textwidth}{
        >{\raggedright\arraybackslash}p{2.2cm}  
        >{\raggedright\arraybackslash}p{1.2cm}  
        >{\raggedright\arraybackslash}p{1.2cm}  
        >{\raggedright\arraybackslash}X         
        c c c                                   
    }
    \toprule
    \textbf{Question} & \textbf{Answer} & \textbf{Low-T} & \textbf{Responses Cluster} & \textbf{SE} & \textbf{KLE} & \textbf{LGU} \\
    \midrule

    Which Greek God gave Cassandra the gift of prophecy? & 
    apollo & 
    Apollo & 
    $\bullet$ the god of prophecy, apollo \newline
    $\bullet$ Apollo \newline
    $\bullet$ the greek God Apollo & 
    0.17 & 
    0.89 & 
    \textbf{0.00} \\
    \midrule

    Who is the actor that plays Ducky on NCIS? & 
    David McCallum & 
    Mark Harmon & 
    $\bullet$ Michael Weatherly \newline
    $\bullet$ Pauley Perrette \newline
    $\bullet$ Mark Harmon \newline
    $\bullet$ Lauren Holly \newline
    $\bullet$ Richard Madden & 
    0.61 & 
    1.50 & 
    1.23 \\
    \midrule

    How many times has Lake Placid hosted the Winter Olympics? & 
    twice & 
    4 times & 
    $\bullet$ Four times (1932, 1980, 1984, and 2010) \newline
    $\bullet$ This is the second time the Winter Olympic Games have been hosted in Lake Placid. \newline
    $\bullet$ 4 times (1999, 2002, 2006, 2010) \newline
    $\bullet$ three times (1932, 1980, and 2002) \newline
    $\bullet$ four times \newline
    $\bullet$ three times (1932, 1980, and 1984). \newline
    $\bullet$ Four times (1998, 2010, 2022, and 2032) \newline
    $\bullet$ 3 times \newline
    $\bullet$ 5 \newline
    $\bullet$ Ten times (twice in the 1930s and 2002, three times in the 1980s, and five times in the 21st century) & 
    2.26 & 
    2.27 & 
    \textbf{3.42} \\
    \bottomrule
    \end{tabularx}
\end{table*}

\end{document}